\documentclass[11pt,a4paper]{article}
\usepackage{amsmath,amssymb,amsthm}
\usepackage{graphicx}
\usepackage{booktabs}
\usepackage{array}
\usepackage{xcolor}
\usepackage{hyperref}
\usepackage[utf8]{inputenc}
\usepackage[T1]{fontenc}
\usepackage{microtype}
\usepackage{natbib}
\usepackage{url}

\newtheorem{theorem}{Theorem}[section]
\newtheorem{lemma}[theorem]{Lemma}
\newtheorem{proposition}[theorem]{Proposition}
\newtheorem{corollary}[theorem]{Corollary}
\newtheorem{definition}[theorem]{Definition}
\newtheorem{remark}[theorem]{Remark}

\newtheorem{assumption}[theorem]{Assumption}

\usepackage{setspace}
\usepackage{mathtools}
\usepackage[utf8]{inputenc}

\newcommand{\argmin}{\mathop{\mathrm{argmin\,}}}
\newcommand{\argmax}{\mathop{\mathrm{argmax\,}}}

\newcommand{\mathbbR}{\mathbb{R}}

\newcommand{\boldone}{{\boldsymbol{1}}}

\newcommand{\boldH}{{\boldsymbol{H}}}
\newcommand{\boldI}{{\boldsymbol{I}}}

\newcommand{\boldK}{{\boldsymbol{K}}}
\newcommand{\boldL}{{\boldsymbol{L}}}

\newcommand{\boldtheta}{{\boldsymbol{\theta}}}

\newcommand{\calD}{{\mathcal{D}}}

\newcommand{\calK}{{\mathcal{K}}}
\newcommand{\calL}{{\mathcal{L}}}

\newcommand{\calS}{{\mathcal{S}}}

\newcommand{\calX}{{\mathcal{X}}}
\newcommand{\calY}{{\mathcal{Y}}}

\newcommand{\Pb}{{P^{\ast}}'}

\usepackage{xcolor}
\usepackage[titletoc]{appendix}
\usepackage{graphicx,graphics}
\usepackage{subcaption}
\usepackage{tikz}
\usepackage{supertabular}
\usepackage{amsmath,amscd,amsfonts,amssymb,hyperref}
\usepackage{dsfont}
\usepackage{fancyhdr}
\usepackage{color}
\usepackage{verbatim}
\usepackage{stmaryrd}
\usepackage{array,supertabular}
\usepackage{float}
\usepackage{rotating}
\usepackage{lscape}
\usepackage{algorithm}
\usepackage{algpseudocode}
\usepackage{multirow}
\usepackage{adjustbox}

\hypersetup{
    colorlinks=true,
    linkcolor=blue,
    citecolor=blue
    }
\usepackage{etoolbox}
\usepackage{booktabs,caption}

\usepackage{dcolumn}
\newcolumntype{L}{D{.}{.}{2,5}}

\def\mds{\medskip}
\def\Rb{{\mathbb R}}

\def\Hc{{\mathcal H}}

\long\def\symbolfootnote[#1]#2{\begingroup%
\def\thefootnote{\fnsymbol{footnote}}\footnotetext[#1]{#2}\footnotemark[#1]\endgroup}
\makeatletter
\@addtoreset{equation}{section}
\renewcommand{\theequation}{\thesection.\arabic{equation}}
\makeatother

\newcounter{Fig}[figure]

\newcounter{Tab}[table]

  {\refstepcounter{Tab}
   \addtocounter{table}{1}
   \samepage\vspace{0.2cm}
   \centerline{#1} \nobreak
   \begin{verse}{Table \thesection.\arabic{Tab}:}
  }%
  {\end{verse} \vspace{0.2cm}}

\relax

\newcommand{\xx}{\mbox{\boldmath $x$}}

\newcommand{\KK}{\mbox{\boldmath $K$}}

\newcommand{\pp}{\mbox{\boldmath $p$}}

\newcommand{\uu}{\mbox{\boldmath $u$}}
\newcommand{\vvv}{\mbox{\boldmath $v$}}

\def \Eb{{\mathbb E}}

\def \Gb{{\mathbb G}}

\def \Jb{{\mathbb J}}
\def \Lb{{\mathbb L}}

\def \Nb{{\mathbb N}}
\def \Pb{{\mathbb P}}
\def \Rb{{\mathbb R}}

\def \Ac{{\mathcal A}}

\def \Bc{{\mathcal B}}
\def \Gc{{\mathcal G}}

\def \Ec{{\mathcal E}}

\def \Sc{{\mathcal S}}
\def \Hc{{\mathcal H}}

\def \Xc{{\mathcal X}}
\def \Yc{{\mathcal Y}}
\def \Ic{{\mathcal I}}

\newcommand{\bqa}{\begin{eqnarray*}}
\newcommand{\eqa}{\end{eqnarray*}}
\newcommand{\bqan}{\begin{eqnarray}}
\newcommand{\eqan}{\end{eqnarray}}
\newcommand{\bqt}{\begin{quote}}
\newcommand{\eqt}{\end{quote}}
\newcommand{\bt}{\begin{tabbing}}
\newcommand{\et}{\end{tabbing}}
\newcommand{\bit}{\begin{itemize}}
\newcommand{\eit}{\end{itemize}}
\newcommand{\ben}{\begin{enumerate}}
\newcommand{\een}{\end{enumerate}}
\newcommand{\beq}{\begin{equation}}
\newcommand{\eeq}{\end{equation}}
\newcommand{\beqw}{\begin{equation*}}
\newcommand{\eeqw}{\end{equation*}}
\newcommand{\bdefi}{\begin{definition}}
\newcommand{\edefi}{\end{definition}}
\newcommand{\bpro}{\begin{proposition}}
\newcommand{\epro}{\end{proposition}}
\newcommand{\blem}{\begin{lemma}}
\newcommand{\elem}{\end{lemma}}
\newcommand{\bco}{\begin{corollary}}
\newcommand{\eco}{\end{corollary}}
\newcommand{\bdes}{\begin{description}}
\newcommand{\edes}{\end{description}}

\newcommand{\eps}{\epsilon}

\def\mds{\medskip}
\def\1{{\mathbf 1}}
\def\0{{\mathbf 0}}

\long\def\acks#1{\vskip 0.3in\noindent{\large\bf Acknowledgments and Disclosure of Funding}\vskip 0.2in
\noindent #1}

\usepackage{xr}
\makeatletter
\newcommand*{\addFileDependency}[1]{% argument=file name and extension
  \typeout{(#1)}
  \@addtofilelist{#1}
  \IfFileExists{#1}{}{\typeout{No file #1.}}
}
\makeatother

\usepackage[top=2.5cm, bottom=2.5cm, left=2.5cm, right=2.5cm]{geometry}
\usepackage{setspace}
\title{Sparse minimum Redundancy Maximum Relevance for feature selection}

\usepackage{authblk}

\author[1]{Peter Naylor}
\author[1,2,3]{Benjamin Poignard}
\author[1]{Héctor Climente-González}
\author[1,3]{Makoto Yamada}

\affil[1]{High-Dimensional Statistical Modeling Team, RIKEN AIP, Kyoto, 606-8501, Japan}
\affil[2]{Keio University, Faculty of Science and Technology, Department of Mathematics, 3-14-1 Hiyoshi, Kohoku-ku, Yokohama, 2238522, Japan}
\affil[4]{Machine Learning and Data Science Unit, Okinawa Institute of Science and Technology, Okinawa, 904-0412, Japan}

\date{\today}
\date{\today}

\hypersetup{
    colorlinks=true,
    linkcolor=blue,
    filecolor=magenta,      
    urlcolor=cyan,
    citecolor=blue,
    pdftitle={Sparse minimum Redundancy Maximum Relevance for feature selection},
    pdfauthor={P. Naylor, B. Poignard, H. Climente-González, M. Yamada}
}

\begin{document}

\maketitle

\begin{abstract}
We propose a feature screening method that integrates both feature-feature and feature-target relationships. Inactive features are identified via a penalized minimum Redundancy Maximum Relevance (mRMR) procedure, which is the continuous version of the classic mRMR penalized by a non-convex regularizer, and where the parameters estimated as zero coefficients represent the set of inactive features. We establish the conditions under which zero coefficients are correctly identified to guarantee accurate recovery of inactive features. We introduce a multi-stage procedure based on the knockoff filter enabling the penalized mRMR to discard inactive features while controlling the false discovery rate (FDR). Our method performs comparably to HSIC-LASSO but is more conservative in the number of selected features. It only requires setting an FDR threshold, rather than specifying the number of features to retain. The effectiveness of the method is illustrated through simulations and real-world datasets. The code to reproduce this work is available on the following GitHub: \url{https://github.com/PeterJackNaylor/SmRMR}.

% Please include a maximum of seven keywords
\end{abstract}

\section{Introduction}

A much employed approach to solve high-dimensional problems is sparse modelling, which focuses on variable selection by discarding unimportant variables for prediction. It aims to recover the true signal when the underlying model admits a sparse representation by applying a penalization procedure to the objective function. A flourishing literature has been dedicated to the devise of suitable penalty functions that can best fit the observed patterns and that satisfy desirable properties: in particular, see \cite{fan2001variable} and \cite{fan2004} for asymptotic results; see \cite{loh2017} or \cite{poignard2022_aism} for a non-asymptotic analysis.

When the data is ultra high-dimensional, that is $p \gg n$ with $p$ the number of variables and $n$ the sample size, sparsity-based methods may not perform well due to their computational cost, lack in both statistical accuracy and algorithmic stability. In light of these issues, \cite{fan2008sure} developed the sure independent screening (SIS) approach. The latter is based on marginal Pearson correlation learning between the features and the target variable. However, the procedure relies on a linear model with Gaussian predictors and responses so that it may not be robust to model misspecification. This gave rise to a broad range of studies on model-free SIS methods: the Distance Covariance of \cite{rizzo2007} and \cite{rizzo2009}; the Distance Correlation of \cite{li2012feature}; the $\sup$-HSIC of \cite{balasubramanian2013ultrahigh}; the Kolmogorov filter of \cite{mai2013} and \cite{mai2015}; the Projection Correlation of \cite{liu2022}; and the mixture Kendall’s Tau and Spearman Rho of \cite{poignard2022}, among others. As pointed out by \cite{mai2015}, feature screening and sparsity-based selection methods differ in their objectives: on the one hand, variable selection attempts to recover the true sparse signal, which may require high-level conditions such as the incoherence assumption; whereas on the other hand, feature screening aims to discover a majority of inactive features and it is thus less ambitious and requires weaker assumptions. 

Furthermore, the technique employed to discover the features truly associated with the target variable should result in a limited number of false discoveries, that is, the False Discovery Rate (FDR) should be low. The selected model should be as large as possible while controlling the FDR, that is the type S error introduced by \cite{gelman2000type}. To satisfy this constraint in linear regression analysis, \cite{barber2015} developed the knockoff filter, which creates copies of the original variables serving as control variables to guarantee the control of the FDR by the feature selection technique. In addition to  upper bounding the directional FDR as well as the standard FDR, it allows to report the selected features for a given confidence level. This novel technique has been extended in \cite{barber2019} to accommodate the $p \gg n$ case and to deal with the model-free viewpoint, as in, e.g., \cite{romano2020}. In particular, the pertinence of the method has been illustrated in the context of biological data: \cite{srinivasan2021compositional} proposed a two-step compositional knockoff filter to control the FDR in microbiome data; \cite{tian2022grace} developed Grace-AKO for graph-constrained estimation and improved the feature selection stability of gene selection by repeating the knockoff procedure and aggregating information.

In this paper, we propose a two step approach for feature screening. After screening out the inactive features, the second step aims to upper bound the FDR at a pre-specified level. In contrast with the standard SIS procedure, which only considers the feature-target variable relationship, our approach refines the screening step by incorporating the feature-feature relationship. This procedure builds upon the minimum redundancy maximum relevance (mRMR) of \cite{PAMI:Peng+etal:2005} and the HSIC-LASSO of \cite{yamada2014high} to solve the redundancy issue. More precisely, we consider the following problem: given $n$ observations of $p$ features $X_1,\ldots,X_p$, we screen out the features whose marginal utility with the target variable $Y$ is not significant. To do so, we develop a novel screening procedure, called sparse mRMR (SmRMR hereafter), which extends the standard mRMR algorithm through the introduction of a non-convex penalization. Indeed, the SmRMR criterion can be recast as a penalized M-estimation problem, where the mRMR parameters and their zero entries correspond to the relaxation parameters of the original mRMR problem and to non-significant dependence, respectively. The optimal threshold selection for screening out the features with the SmRMR is performed via cross-validation, contrary to model-free screening procedures. In the second step, we apply the knockoff filter to the features selected in the first step, allowing us to report the selected features for a given confidence level. Our main contributions are as follows: we establish the existence of a consistent estimator for the SmRMR problem and show that the true zero elements of the mRMR criterion can be consistently estimated, ensuring the correct identification of the non-active features; conditional on the screening step, using the knockoff filter on the selected features, we verify that our SmRMR method controls the FDR for a given confidence level. 

The remainder of the paper is organized as follows. Section \ref{prelim} details the framework and the SmRMR procedure. Section \ref{asymptotic_properties} is devoted to the theoretical properties of the SmRMR. Section \ref{mRMR_fdr} details the knockoff filter for SmRMR and FDR control. Section \ref{applications} illustrates the performances of the proposed method through simulations and real data experiments. All the proofs of the
main text and auxiliary results are relegated to the Appendix.

\textbf{\emph{Notations.}} Throughout this paper, we denote the cardinality of a set $E$ by $\text{card}(E)$. For a vector $\vvv \in \Rb^{d}$, the $\ell_p$ norm is $\|\vvv\|_p = \big(\sum^{\text{d}}_{k=1} |\vvv_k|^p \big)^{1/p}$ for $p > 0$, and $\|\vvv\|_{\infty} = \max_i|\vvv_i|$. 
%Let the subset $\Sc \subseteq \{1,\ldots,d\}$, then $\vvv_{\Sc} \in \Rb^{|\Sc|}$ is the vector $\vvv$ restricted to $\Sc$. 
For a symmetric matrix $A$, $\lambda_{\min}(A)$ (resp. $\lambda_{\max}(A)$) denotes the minimum (resp. maximum) eigenvalue of $A$, and $\text{tr}(A)$ is the trace operator. For a matrix $B$, $\|B\|_s = \lambda^{1/2}_{\max}(B^\top B)^{1/2}$ and $\|B\|_F=\text{tr}(B^\top B)$ are the spectral and Frobenius norms, respectively, and $\|B\|_{\max} = \max_{ij}|B_{i,j}|$ is the coordinate-wise maximum. 
For a function $f: \Rb^{d} \rightarrow \Rb$, we denote by $\partial f$ the component-by-component partial derivative of $f$. We write $\Sc^c$ to denote the complement of the set $\Sc$. 

\section{Framework} \label{prelim}

\subsection{Feature screening}\label{feature_screening_framework}

Standard feature screening methods consider a response variable $Y$ and $p$ features $(X_1,\ldots,X_p)$ among which they aim to discover the inactive ones through a marginal utility $\text{D}(.,.)$ between $Y$ and $X_k, 1 \leq k \leq p$. The viewpoint is the Sure Independence Screening (SIS). In the same vein as in \cite{li2012feature}, let $\Sc$ be the set of active features and $\Ic$ be the set of inactive features, respectively defined as
\begin{equation*}
\begin{array}{llll}
\Sc &:=& \big\{k: \, \text{F}(Y|X_1,\ldots,X_p) \;\text{functionally depends on} \; X_k, 1 \leq k \leq d\big\},\\
\Ic &:=& \big\{k: \, \text{F}(Y|X_1,\ldots,X_p) \; \text{does not functionally depend on} \; X_k, 1 \leq k \leq d\big\},
\end{array}
\end{equation*}
where $\text{F}(Y|X_1,\ldots,X_p)$ is the probability distribution of $Y|X_1,\ldots,X_p$. The SIS approach aims to find a majority of $\Ic$, that is given the set of active features $\{X_k, k \in \Sc\}$, it checks that the features $\{X_k, k \in \Ic\}$ are independent of $Y$. To estimate $\Sc$ and screen out $\{X_k, k \in \Ic\}$, SIS computes an estimator of the marginal utility $\text{D}(.,.)$, denoted by $\widehat{\text{D}}_v(.,.)$, which quantifies the dependence between $Y$ and feature $X_k$. More formally, the estimator $\widehat{\text{D}}_v(\mathbf{X}_k,\mathbf{Y})$ is deduced from the sample of observations $\mathbf{Y}=(Y_1,\ldots,Y_n)$ and $\mathbf{X}_k = (X_{1k},\ldots,X_{nk}), k = 1,\ldots,p$, and the theoretical set $\Sc$ is estimated by the set
\begin{equation}\label{active_var_estim}
\widehat{\Sc}^{\lambda_n} := \big\{k: \; \widehat{\text{D}}_v(\mathbf{X}_k,\mathbf{Y}) \geq \lambda_n, \; k = 1,\ldots,p\big\},
\end{equation}
where $\lambda_n \geq 0$ is a threshold parameter that depends on the sample size and controls the number of selected active features. The calibration of $\lambda_n$ is key to obtain the sure independence screening (SIS) property of the procedure, that is recovering with high probability the set $\Sc$. Here, $\widehat{\text{D}}_v(.,.)$ is a V-statistic estimator, non-negative, of the marginal utility $\text{D}(.,.)$.

\subsection{Minimal-redundancy Maximal-relevance (mRMR)}

Aiming to the same objective of finding a set $S$ of size $s$ of active features, an alternative approach is the maximum relevance (MR) feature selection developed by \cite{PAMI:Peng+etal:2005}. Let $\widehat{\text{D}}_v(\mathbf{X}_k,\mathbf{Y})$ be the V-statistic estimator of $\text{D}(.,.)$, then MR aims to maximize the relevance $(\text{card}(S))^{-1}\sum_{k \in \calS}\widehat{\text{D}}_v(\mathbf{X}_k,\mathbf{Y})$.
% \begin{equation*}
%     \widehat{\Sc} = \argmax_{\calS} \;\; \frac{1}{s}\sum_{k \in \calS}\widehat{\text{D}}_v(\mathbf{X}_k,\mathbf{Y}).
% \end{equation*}
In their seminal work, \cite{PAMI:Peng+etal:2005} use the Mutual Information as the marginal utility. One drawback of the MR method is the tendency to select redundant features, that is the selected features can be highly correlated since it does not use the feature-feature relationship. This is a drawback also shared by the SIS approach. Redundant features refers to the situation where adding highly correlated variables does not improve information. Since no feature-feature relationship is considered in both the SIS and MR approaches, they are unable to discover which combination of features would result in the best performance: see \cite{PAMI:Peng+etal:2005} or \cite{guyon2003} for further details on the redundancy issue.

To overcome the redundancy problem, \cite{PAMI:Peng+etal:2005} devised the mRMR feature selection algorithm, which estimates the set of active features as
\begin{equation}
\label{eq:mrmr}
\widehat{\Sc} = \argmax_{\calS} \; \frac{1}{s}\sum_{k \in \calS}\widehat{\text{D}}_v(\mathbf{X}_k,\mathbf{Y}) \!-\! \frac{1}{s^2}\sum_{k \in \calS} \sum_{l \in \calS}\!\widehat{\text{D}}_v(\mathbf{X}_k,\mathbf{X}_{l}).
\end{equation}
The second term in (\ref{eq:mrmr}) can be viewed as a ``penalization'' term that selects independent features: this is the redundancy term. Since $\widehat{\text{D}}_v(\mathbf{X}_k,\mathbf{X}_{l})$ takes non-negative values when $X_k$ and $X_{l}$ are non-independent, the second term takes large negative values if the selected features are not mutually independent. Thus, by selecting features through the maximization of equation (\ref{eq:mrmr}), we can select features that are dependent on the output, and the selected features are mutually independent. This criterion selects the features with both minimum redundancy, approximated as the mean value of the association measure between each pair of variables in the subset, and maximum relevance, estimated as the mean value of the association measure between each feature and the target class. 

\subsection{Penalized mRMR}

The original mRMR algorithm consists of a discrete optimization problem, and the optimization is in general difficult so that continuous optimization tends to be used to mitigate the problem. The mRMR algorithm can be recast as
\begin{equation*}
\widehat{\boldsymbol{\beta}} = \underset{\mathbf{\beta} \in \{0,1\}^p}{\argmin} \;\; -\frac{1}{\mathbf{\beta}^\top \boldone_p}\sum_{k =1}^p \beta_k \widehat{\text{D}}_v(\mathbf{X}_k,\mathbf{Y}) + \frac{1}{(\mathbf{\beta}^\top \boldone_p)^2}\sum_{k,l =1}^p \beta_k \beta_{l}\widehat{\text{D}}_v(\mathbf{X}_k,\mathbf{X}_{l}).
\end{equation*}
We then relax $\mathbf{\beta}$ as $\mathbf{\theta} = (\theta_1,\theta_2, \ldots, \theta_p)^\top \in \mathbbR^p$ so that the problem becomes
\begin{equation*}
\widetilde{\mathbf{\theta}} = \underset{\mathbf{\theta} \in \Rb^p_+ }{\argmin} \; -\sum_{k = 1}^p \theta_k\widehat{\text{D}}_v(\mathbf{X}_k,\mathbf{Y}) + \frac{1}{2} \sum_{k,l = 1}^p\theta_k \theta_{l}\widehat{\text{D}}_v(\mathbf{X}_k,\mathbf{X}_{l}).
\end{equation*}
The problem is equivalent to minimizing, up to some constant terms, the loss function
\begin{equation} \label{unpen_loss}
\Lb_{v,n}(\mathbf{\theta}) = \frac{1}{2}\widehat{\text{D}}_v(\mathbf{Y},\mathbf{Y}) -  \overset{p}{\underset{k=1}{\sum}} \theta_k \widehat{\text{D}}_v(\mathbf{X}_k,\mathbf{Y}) + \frac{1}{2} \overset{p}{\underset{k,l=1}{\sum}} \theta_k\theta_l \widehat{\text{D}}_v(\mathbf{X}_k,\mathbf{X}_l).
\end{equation}
To perform feature screening while taking into account the feature-feature relationship, we apply a penalty term to $\Lb_{v,n}(\theta)$ in the minimization problem so that the problem becomes
\begin{equation}\label{penalized_mrmr}
\widehat{\mathbf{\theta}} \in \underset{\mathbf{\theta} \in \Rb^p_+}{\argmin} \; n\,\Lb_{v,n}(\mathbf{\theta}) + n\overset{p}{\underset{k=1}{\sum}}\pp(\lambda_n,\theta_k),
\end{equation}
when such a minimizer exists, where $\pp(\lambda_n,x), x \geq 0$ is a penalty function that fosters sparsity among $\mathbf{\theta}$ with regularization parameter $\lambda_n \geq 0$. Problem (\ref{penalized_mrmr}) is our proposed SmRMR procedure. In this paper, we consider the LASSO penalty of \cite{tibshirani1996}, defined as $\pp(\lambda,x) = \lambda x$ for $x \geq 0$, and the non-convex penalty functions SCAD and MCP. The SCAD penalty of \cite{fan2001variable} is defined as: for every $x \geq 0$,
\begin{equation*}
\pp(\lambda,x) = \lambda x \, \mathbf{1}\big(x \leq \lambda\big) + \frac{2 a_{\text{scad}}\lambda x-x^2-\lambda^2}{2(a_{\text{scad}}-1)}\, \mathbf{1}\big(\lambda < x \leq a_{\text{scad}} \lambda\big) + \frac{1}{2}(a_{\text{scad}}+1)\lambda^2 \, \mathbf{1}\big(x > a_{\text{scad}}\lambda\big),
\end{equation*}
where $a_{\text{scad}}>2$ and whose derivative is $\partial_x \pp(\lambda,x) = \lambda \boldone_{(x\leq \lambda)} + \frac{(a_{\text{scad}}\lambda-x)^+}{(a_{\text{scad}}-1)}$.
The MCP due to~\cite{zhang2010} is defined for $b_{\text{mcp}}>0$ as: for every $x\geq 0$, 
\begin{equation*}
\pp(\lambda,x)  = \big(\lambda x-\frac{x^2}{2 b_{\text{mcp}}}\big) \boldone\big(x\leq b_{\text{mcp}}\lambda\big) + \lambda^2\frac{b_{\text{mcp}}}{2} \boldone
\big(x> b_{\text{mcp}}\lambda\big),
\end{equation*}
and whose derivative is $\partial_x \pp(\lambda,x) = (\lambda - \frac{x}{b_{\text{mcp}}})^+$.
Note that when $a_{\text{scad}} \rightarrow \infty$ and $b_{\text{mcp}} \rightarrow \infty$, then the SCAD and MCP become the LASSO. In light of (\ref{unpen_loss}) and (\ref{penalized_mrmr}), we observe the following properties: a strong dependence between $Y$ and the $k$-th feature $X_k$ implies a large positive value for $\widehat{\text{D}}_v(\mathbf{X}_k,\mathbf{Y})$, so that $\theta_k$ should be estimated as a large positive value in (\ref{penalized_mrmr}). On the contrary, a weak relationship should translate into small values of $\widehat{\text{D}}_v(\mathbf{X}_k,\mathbf{Y})$, so that the corresponding parameter $\theta_k$ should be shrunk to zero by the penalty function. Finally, when the $k$-th and $l$-th features $X_k, X_l$ are strongly correlated, then $\widehat{\text{D}}_v(\mathbf{X}_k,\mathbf{X}_l)$ takes a large positive value, so that $\theta_k$ or $\theta_l$ should be shrunk to zero in (\ref{penalized_mrmr}) for minimization. Problem (\ref{penalized_mrmr}) is actually a modification of the standard feature screening method (\ref{active_var_estim}). Consider the first order condition (righ-derivative) of problem (\ref{penalized_mrmr}) with the LASSO penalty. We get
\begin{equation*}
\forall 1 \leq k \leq p,\; -n\widehat{\text{D}}_v(\mathbf{X}_k, \mathbf{Y}) + n\sum_{l = 1}^p \widehat\theta_{l} \widehat{\text{D}}_v(\mathbf{X}_k,\mathbf{X}_l)  + n\partial_{\theta_k}\big(\pp(\lambda_n,\theta_k) \big) = 0,
\end{equation*}
where $\partial_{\theta_k}\big(\pp(\lambda_n,|\theta_k|) \big)= \lambda_n \widehat{\mathbf{z}}_{k}$, with $\widehat{\mathbf{z}}_{k}$ the subdifferential defined as
\begin{equation*}
\widehat{\mathbf{z}}_{k} = \text{sgn}(\widehat{\theta}_{k}) = 1 \;\; \text{when} \;\; \widehat{\theta}_{k} > 0, \;\; \text{and} \;\; \widehat{\mathbf{z}}_{k} \in \big\{\widehat{\mathbf{z}}_k:\widehat{\mathbf{z}}_{k}\leq 1\big\} \;\; \text{when} \;\; \widehat{\theta}_{k}=0,
\end{equation*}
under the non-negative constraint. It can be noted that the subdifferential of the LASSO under the non-negative constraint is similar to the standard LASSO. Now consider the case $\widehat{\theta}_{k} > 0, 1 \leq k \leq p$. The orthogonality condition becomes $\widehat{\text{D}}_v(\mathbf{X}_k,\mathbf{Y}) =\sum_{l= 1}^p \widehat\theta_{l} \widehat{\text{D}}_v(\mathbf{X}_k,\mathbf{X}_l) + \lambda_n$.
Under the non-negative constraint on the parameters and since the V-statistic estimator of $\text{D}(.,.)$ is non-negative, $\sum_{l = 1}^p \widehat\theta_{l} \widehat{\text{D}}_v(\mathbf{X}_k,\mathbf{X}_l) \geq 0$ so that $\widehat{\text{D}}_v(\mathbf{X}_k,\mathbf{Y}) \geq \lambda_n$ by the orthogonality condition. Now, when $\widehat{\theta}_{k}=0, 1 \leq k \leq p$, then
\begin{equation*}
\widehat{\text{D}}_v(\mathbf{X}_k,\mathbf{Y})=\sum_{l\neq k}^p \widehat\theta_{l} \widehat{\text{D}}_v(\mathbf{X}_k,\mathbf{X}_l) + \lambda_n\widehat{\mathbf{z}}_{k}, \;\text{so that} \; \widehat{\text{D}}_v(\mathbf{X}_k,\mathbf{Y})\leq \sum_{l \neq k}^p \widehat\theta_{l} \widehat{\text{D}}_v(\mathbf{X}_k,\mathbf{X}_l) +  \lambda_n.
\end{equation*}
Hence, the sparse estimation of $\mathbf{\theta}$ under the non-negative constraint allows to perform feature screening in the same vein as in the standard procedure for feature screening. More precisely, a positive estimated parameter $\widehat\theta_k$ implies $\widehat{\text{D}}_v(\mathbf{X}_k,\mathbf{Y})\geq \lambda_n$: from this inequality, the standard feature screening procedure identifies $k$ as an important feature for $Y$ in light of the set (\ref{active_var_estim}). Now, when the $k$-th component of $\mathbf{\theta}$ is shrunk to zero by the penalty function, then $\widehat{\text{D}}_v(\mathbf{X}_k,\mathbf{Y}) \leq \sum_{l \neq k}^p \widehat\theta_{l} \widehat{\text{D}}_v(\mathbf{X}_k,\mathbf{X}_l)+\lambda_n$. Standard feature screening screens out an unimportant feature when $\widehat{\text{D}}_v(\mathbf{X}_k,\mathbf{Y})<\lambda_n$. Thus, our proposed methodology refines the standard approach by incorporating the feature-feature relationship in the decision to select/exclude the feature.

In the SCAD case, when $0<\widehat{\theta}_{k}\leq \lambda_n$, the problem is similar to the LASSO shrinkage. If $\widehat{\theta}_{k}>\lambda_n\neq 0, 1 \leq k \leq p$ (coefficient sufficiently large), the first order condition becomes $\widehat{\text{D}}_v(\mathbf{X}_k,\mathbf{Y}) = \sum_{l = 1}^p \widehat\theta_{l} \widehat{\text{D}}_v(\mathbf{X}_k,\mathbf{X}_l) + \frac{(a_{\text{scad}}\lambda_n-\widehat{\theta}_{k})^+}{(a_{\text{scad}}-1)}$.
When $\lambda_n<\widehat{\theta}_{k}<a_{\text{scad}}\lambda_n$, then the selection rule is $\widehat{\text{D}}_v(\mathbf{X}_k,\mathbf{Y})\geq \frac{a_{\text{scad}}\lambda_n-\widehat{\theta}_{k}}{(a_{\text{scad}}-1)}\geq \lambda_n-\frac{\widehat{\theta}_{k}}{a_{\text{scad}}-1}\cdot$ When $a_{\text{scad}}\rightarrow \infty$, then the selection rule is the same as in the LASSO case, since the SCAD penalty tends to the LASSO penalty. In the same manner, when the penalty is the MCP, when $0<\widehat{\theta}_{k}<b_{\text{mcp}}\lambda_n$, then the selection rule is $\widehat{\text{D}}_v(\mathbf{X}_k,\mathbf{Y})\geq \lambda_n-\frac{\widehat{\theta}_{k}}{b_{\text{mcp}}}\cdot$

\subsection{Association measures}\label{association_measure}

Solving problem (\ref{penalized_mrmr}) requires the specification of an association measure $\text{D}(.,.)$ together with its estimator $\widehat{\text{D}}_v(.,.)$ to characterize the dependence between two random variables, say $X$ and $Y$. In this paper, we will consider the following two commonly used measures: the squared Projection Correlation ($\text{PC}^2$); the squared Hilbert-Schmidt norm of the cross-covariance operator ($\text{HSIC}$), which includes the Distance Covariance ($\text{Dcov}$) for a suitable kernel. These measures share common properties when measuring the dependence: both $\text{HSIC}$ and $\text{PC}$ are non-negative and equal zero if and only if the two arguments are independent. Hereafter, the two random variables $X,Y$ take values on $(\Xc,\Bc_{x})$ and $(\Yc,\Bc_{y})$, respectively, where $\Xc,\Yc$ are two separable metric spaces, $\Bc_{x},\Bc_{y}$ are Borel $\sigma$-algebras. Then, $(\Xc \times \Yc, \Bc_{x} \times \Bc_{y})$ is measurable, and the joint distribution is defined as $\Pb_{X \, Y}$, which assigns values to the product space $(\Xc \times \Yc, \Bc_{x} \times \Bc_{y})$. 

\subsubsection{Projection Correlation ($\text{PC}$)}

The Projection Correlation of $X$ and $Y$ of \cite{zhu2017}, denoted by $\text{PC}(X,Y)$, is defined as the square root of $\text{PC}(X,Y)^2 = \frac{\text{Pcov}(X,Y)^2}{\text{Pcov}(X,X)\text{Pcov}(Y,Y)}$, 
where the squared Projection Covariance $\text{Pcov}(X,Y)^2$ is:
\begin{align*}
\text{Pcov}(X,Y)^2 &= \Eb\big[\text{arccos}\big(\frac{(X_1-X_3)^\top(X_4-X_3)}{\|X_1-X_3\|\|X_4-X_3\|}\big)\text{arccos}\big(\frac{(Y_1-Y_3)^\top(Y_4-Y_3)}{\|Y_1-Y_3\|\|Y_4-Y_3\|}\big)\big]\\
&  + \Eb\big[\text{arccos}\big(\frac{(X_1-X_3)^\top(X_4-X_3)}{\|X_1-X_3\|\|X_4-X_3\|}\big)\text{arccos}\big(\frac{(Y_2-Y_3)^\top(Y_5-Y_3)}{\|Y_2-Y_3\|\|Y_5-Y_3\|}\big)\big] \\
&  -2 \,\Eb\big[\text{arccos}\big(\frac{(X_1-X_3)^\top(X_4-X_3)}{\|X_1-X_3\|\|X_4-X_3\|}\big)\text{arccos}\big(\frac{(Y_2-Y_3)^\top(Y_4-Y_3)}{\|Y_2-Y_3\|\|Y_4-Y_3\|}\big)\big] \\ &=: S_1+S_2-2S_3,
\end{align*}
with $(X_1,Y_1),\ldots,(X_5,Y_5)$ independent copies of $(X,Y)$. By construction, $\text{Pcov}(X,Y)^2$ does not require moment conditions on $X,Y$ as it relies on vectors $(X_k-X_l)/\|X_k-X_l\|$ and $(Y_k-Y_l)/\|Y_k-Y_l\|$. By proposition 1 of \cite{zhu2017}, $0 \leq \text{PC}(X,Y) \leq 1$ in general, and $\text{PC}(X,Y) = 0$ if and only if $X$ and $Y$ are independent. Equipped with the sample of observations $\mathbf{X} = (X_1,\ldots,X_n)$ and $\mathbf{Y} = (Y_1,\ldots,Y_n)$ from the population $(X,Y)$, $\text{Pcov}(X,Y)^2$ can be estimated by the V-statistic $\widehat{\text{Pcov}}_v(\mathbf{X},\mathbf{Y})^2$ defined as $\widehat{\text{Pcov}}_v(\mathbf{X},\mathbf{Y})^2 = \widehat{S}_{1,v} + \widehat{S}_{2,v} - 2\widehat{S}_{3,v}$,
where 
\begin{align*}
\widehat{S}_{1,v} &=\frac{1}{n^2}\overset{n}{\underset{i,j,l=1}{\sum}}\text{arccos}\big(\frac{(X_i-X_j)^\top(X_l-X_j)}{\|X_i-X_j\|\|X_l-X_j\|}\big)\text{arccos}\big(\frac{(Y_i-Y_j)^\top(Y_l-Y_j)}{\|Y_i-Y_j\|\|Y_l-Y_j\|}\big),\\
\widehat{S}_{2,v}&=\frac{1}{n^5}\overset{n}{\underset{i,j,l,m,r=1}{\sum}} \text{arccos}\big(\frac{(X_i-X_l)^\top(X_m-X_l)}{\|X_i-X_l\|\|X_m-X_l\|}\big)\text{arccos}\big(\frac{(Y_j-Y_l)^\top(Y_r-Y_l)}{\|Y_j-Y_l\|\|Y_r-Y_l\|}\big),\\
\widehat{S}_{3,v}& =\frac{1}{n^4}\overset{n}{\underset{i,j,l,m=1}{\sum}} \text{arccos}\big(\frac{(X_i-X_l)^\top(X_m-X_l)}{\|X_i-X_l\|\|X_m-X_l\|}\big)\text{arccos}\big(\frac{(Y_j-Y_l)^\top(Y_m-Y_l)}{\|Y_j-Y_l\|\|Y_m-Y_l\|}\big).
\end{align*}
By Theorem 1 of \cite{zhu2017}, the estimator $\widehat{\text{Pcov}}_v(\mathbf{X},\mathbf{Y})^2$ of $\text{Pcov}(X,Y)^2$ can be written as
\begin{equation}\label{pc_etim_2}
\widehat{\text{Pcov}}_v(\mathbf{X},\mathbf{Y})^2 = \frac{1}{n^3}\overset{n}{\underset{i,l,r=1}{\sum}} [\calK_X]_{ilr}[\calL_Y]_{ilr},
\end{equation}
with $[\calK_X]_{ilr} = K_{ilr} - \overline{K}_{i.r} - \overline{K}_{.lr} + \overline{K}_{..r}$, $[\calL_Y]_{ilr} = L_{ilr} - \overline{L}_{i.r} - \overline{L}_{.lr} + \overline{L}_{..r}$, where 
\begin{align*}
K_{ilr} &= \text{arccos}\big(\frac{(X_i-X_r)^\top(X_l-X_r)}{\|X_i-X_r\|\|X_l-X_r\|}\big), \overline{K}_{i.r} = \frac{1}{n}\overset{n}{\underset{l=1}{\sum}}K_{ilr},  \overline{K}_{.lr} = \frac{1}{n}\overset{n}{\underset{i=1}{\sum}}K_{ilr},  \overline{K}_{..r} = \frac{1}{n^2}\overset{n}{\underset{i,l=1}{\sum}}K_{ilr},\\
L_{ilr} &= \text{arccos}\big(\frac{(Y_i-Y_r)^\top(Y_l-Y_r)}{\|Y_i-Y_r\|\|Y_l-Y_r\|}\big), \overline{L}_{i.r} = \frac{1}{n}\overset{n}{\underset{l=1}{\sum}}L_{ilr},  \overline{L}_{.lr} = \frac{1}{n}\overset{n}{\underset{i=1}{\sum}}L_{ilr},  \overline{L}_{..r} = \frac{1}{n^2}\overset{n}{\underset{i,l=1}{\sum}}L_{ilr},
\end{align*}
with $K_{ilr}=L_{ilr}=0$ if $i=r$ or $l=r$.
The V-statistic estimator of the squared projection correlation is $\widehat{\text{PC}}_v(\mathbf{X},\mathbf{Y})^2=\widehat{\text{Pcov}}_v(\mathbf{X},\mathbf{Y})^2/(\widehat{\text{Pcov}}_v(\mathbf{X},\mathbf{X})\widehat{\text{Pcov}}_v(\mathbf{Y},\mathbf{Y}))$. 
In this paper, we will implement the estimator defined in (\ref{pc_etim_2}) when using the (squared) projection correlation measure.

\subsubsection{Hilbert-Schmidt Independence Criterion (HSIC)}

$\text{HSIC}$ is a cross-covariance operator in reproducing kernel Hilbert spaces (RKHS). It has fostered a broad range of studies dedicated to feature selection: see \cite{ALT:Gretton+etal:2005}, \cite{song2012feature} for a theoretical analysis of $\text{HSIC}$. Let us first fix the setup to define this measure. Let $\Xc$ be a metric space and $\Hc$ a Hilbert space of functions $f: \Xc \rightarrow \Rb$.  $\Hc$ is a reproducing kernel Hilbert space (RKHS) induced by the inner product $\langle.,.\rangle$ if there exists a function $k: \Xc \times \Xc \rightarrow \Rb$ such that $\forall \xx \in \Xc, \; k(\xx,.) \in \Hc$, and $\forall f \in \Hc, \forall \xx \in \Xc, \langle f , k(\xx,.)\rangle = f(\xx)$. For any probability measure $\Pb$ defined on $\Xc$, the mean $\mu(\Pb) \in \Hc$ is defined as $\Eb[f(X)] = \langle f, \mu(\Pb)\rangle$ for any $f \in \Hc$, where the random variable $X$ is sampled from $\Xc$. Likewise, consider $\Gc$ another RKHS defined on a metric space $\Yc$ with kernel $l(\cdot)$. We then define the symmetric bounded kernels $k(\cdot,\cdot)$ and $l(\cdot,\cdot)$ on the spaces $\Xc$ and $\Yc$, respectively. Now given two separable RKHSs $\Hc$ and $\Gc$ and a joint measure $\Pb_{X\, Y}$ defined over $(\Xc \times \Yc, \Bc_{x} \times \Bc_{y})$, by Lemma 1 of \cite{ALT:Gretton+etal:2005}, $\text{HSIC}$ is defined as:
\begin{align*}
\text{HSIC}(X,Y) &= \Eb_{X_1X_2Y_1Y_2}\big[k(X_1,X_2)l(Y_1,Y_2)\big]\\
&+ \Eb_{X_1X_2}\big[k(X_1,X_2)\big]\Eb_{Y_1Y_2}\big[l(Y_1,Y_2)\big] - 2\Eb_{X_1Y_1}\big[\Eb_{X_2}[k(X_1,X_2)] \Eb_{Y_2}[l(Y_1,Y_2)]\big],
\end{align*}
with $(X_1,Y_1),(X_2,Y_2)$ independent copies of $(X,Y)$. When $X$ and $Y$ are independent, $\text{HSIC}(X,Y)=0$, and $\text{HSIC}(X,Y)>0$ otherwise: see Section 5 in \cite{gretton2005} or Subsection 2.3 in \cite{song2012feature}. 
%Interestingly, one can recover the Distance Covariance measure of \cite{szekely2009brownian} when $k(X_1,X_2) = \frac{1}{2}(\|X_1\|+\|X_2\|+\|X_1-X_2\|)$, $l(Y_1,Y_2) = \frac{1}{2}(\|Y_1\|+\|Y_2\|+\|Y_1-Y_2\|)$ as emphasized by \cite{balasubramanian2013ultrahigh}. 
Equation (4) of \cite{gretton2007kernel} provides the V-statistic estimator
\begin{equation}\label{hsic_v1}
\widehat{\text{HSIC}}_{v}(\mathbf{X},\mathbf{Y}) = \frac{1}{n^2}\overset{n}{\underset{i,j=1}{\sum}} K_{ij}L_{ij} + \frac{1}{n^4} \overset{n}{\underset{i,j,m,l=1}{\sum}} K_{ij}L_{ml} - \frac{2}{n^3} \overset{n}{\underset{i,j,m=1}{\sum}} K_{ij} L_{im},
\end{equation}
with $K_{ij} = k(X_i,X_j)$, $L_{ij} = l(Y_i,Y_j)$. It can be alternatively written as
\begin{equation} \label{hsic_v2}
    \widehat{\text{HSIC}}_{v}(\mathbf{X},\mathbf{Y}) = \frac{1}{n^2} \sum_{i,j= 1}^n [\calK_X]_{ij}[\calL_Y]_{ij} = \frac{1}{n^2} \text{tr}\left(\boldK_{X} \boldH_n\boldL_{Y} \boldH_n\right),
\end{equation}
where $\calK_X = \boldH_n \boldK_{X} \boldH_n$ and $\calL_Y = \boldH_n \boldL_Y \boldH_n$, $[\boldK_X]_{ij} = k(X_i,X_j)$, $[\boldL_Y]_{ij} = l(Y_i,Y_j)$ and $\boldH_n = \boldI_n - \frac{1}{n}\boldone_n \boldone_n^\top$ is the centering matrix with $\boldone_n$ the $n$-dimensional vector containing $1$s only, so that
\begin{align*}
[\calK_X]_{ij} & = K_{ij} - \overline{K}_{.j} - \overline{K}_{i.} - \overline{K}, \overline{K}_{.j} = \frac{1}{n}\overset{n}{\underset{i=1}{\sum}} k(X_i,X_j), \overline{K}_{i.} = \frac{1}{n}\overset{n}{\underset{j=1}{\sum}} k(X_i,X_j), \overline{K} = \frac{1}{n^2}\overset{n}{\underset{i,j=1}{\sum}} k(X_i,X_j),\\
[\calL_Y]_{ij} & = L_{ij} - \overline{L}_{.j} - \overline{L}_{i.} - \overline{L}, \overline{L}_{.j} = \frac{1}{n}\overset{n}{\underset{i=1}{\sum}} l(Y_i,Y_j), \overline{L}_{i.} = \frac{1}{n}\overset{n}{\underset{j=1}{\sum}} l(Y_i,Y_j), \overline{L} = \frac{1}{n^2}\overset{n}{\underset{i,j=1}{\sum}} l(Y_i,Y_j). 
\end{align*}
This V-statistic estimator is biased with a bias of order $O(n^{-1})$ according to Theorem 1 of \cite{ALT:Gretton+etal:2005}. The unbiased version of $\text{HSIC}$ proposed by, e.g., \cite{gretton2007kernel} - see their equation (3) - is defined as
\begin{eqnarray}\label{hsic_u}
\begin{split}
    \widehat{\text{HSIC}}_u(\mathbf{X},\mathbf{Y}) = \frac{1}{n(n-1)}\underset{(i,j)\in\mathbf{i}^n_2}{\sum} K_{ij}L_{ij}+ \frac{1}{n(n-1)(n-2)(n-3)} \underset{(i,j,m,l)\in\mathbf{i}^n_4}{\sum} K_{ij}L_{ml} \\ - \frac{2}{n(n-1)(n-2)} \underset{(i,j,m)\in\mathbf{i}^n_3}{\sum} K_{ij} L_{im},
\end{split}
\end{eqnarray}
with $\mathbf{i}^n_m$ the index set denoting the set of all $m$-tuples drawn without replacement from $\{1,\ldots,n\}$. Hereafter and to clarify our arguments, we consider non-negative symmetric and continuous kernels $k(\cdot,\cdot),l(\cdot,\cdot)$, bounded by $1$. This condition includes several commonly used kernels, such as the Gaussian kernel or the Laplacian kernel. We will consider the Gaussian kernel in our empirical studies. Both of these kernels are universal in the sense of \cite{gretton2005} and \cite{fukumizu2009}.

In this paper, the theoretical and empirical results will be based on the normalized version of HSIC formulated in \cite{daveiga2015} as a sensitivity index, which will be denoted by $\text{nr-HSIC}$, and defined as:
\begin{equation}\label{nr-HSIC}
\text{nr-HSIC}(X,Y) = \cfrac{\text{HSIC}(X,Y)}{\sqrt{\text{HSIC}(X,X)\text{HSIC}(Y,Y)}},
\end{equation}
where $\text{HSIC}(X,X)>0, \text{HSIC}(Y,Y)>0$. Replacing $\text{HSIC}(X,Y), \text{HSIC}(X,X), \text{HSIC}(Y,Y)$ by their estimators based on (\ref{hsic_v2}), the V-statistic estimator of the normalized HSIC is
\begin{equation}\label{nr-HSIC-vstat}
\text{nr-}\widehat{\text{HSIC}}_{v}(\mathbf{X},\mathbf{Y}) = \frac{n^{-2}\text{tr}\left(\boldK_{X} \boldH_n\boldL_{Y} \boldH_n\right)}{\sqrt{n^{-2}\text{tr}\left(\boldK_{X} \boldH_n\boldK_{X} \boldH_n\right)n^{-2}\text{tr}\left(\boldL_{Y} \boldH_n\boldL_{Y}.\boldH_n\right)}}\cdot
\end{equation}
Note that, by the properties of the $\text{tr}(\cdot)$ operator, $\sqrt{\text{tr}\left(\boldK_{X} \boldH_n\boldK_{X} \boldH_n\right)} = \|\boldH_n\boldK_{X} \boldH_n\|_F$.

Furthermore, $\Lb_{v,n}(\mathbf{\theta})$ in (\ref{unpen_loss}) can be written as a least squares criterion in the same vein as in \cite{poignard2020_aistats}:
\begin{itemize}
    \item[(i)] When $\normalfont\widehat{\text{D}}_v(.,.) = \text{nr-}\widehat{\text{HSIC}}_v(.,.)$, the loss becomes
    \begin{equation*}
    \Lb_{v,n}(\mathbf{\theta}) = \frac{1}{2}\;\frac{1}{n^2}\|\overline\calL_Y - \sum_{k = 1}^p \theta_k \overline\calK_{X_k}\|^2_F = \frac{1}{2} \; \frac{1}{n^2}\overset{n}{\underset{i,j=1}{\sum}}\Big( [\overline\calL_Y]_{ij} - \sum_{k = 1}^p \theta_k [\overline\calK_{X_k}]_{ij}\Big)^2,
    \end{equation*}
    where $[\overline\calL_Y]_{ij} = [\calL_Y]_{ij}/(n^{-1}\|\boldH_n\boldL_{Y} \boldH_n\|_F)$ and $[\overline\calK_{X_k}]_{ij} = [\calK_{X_k}]_{ij}/(n^{-1}\|\boldH_n\boldK_{X_k} \boldH_n\|_F)$.
    \item[(ii)] When $\normalfont\widehat{\text{D}}_v(.,.) = \widehat{\text{PC}}_v(.,.)^2$, the loss becomes
    \begin{equation*}
    \Lb_{v,n}(\mathbf{\theta}) = \frac{1}{2}\;\frac{1}{n^3} \langle \overline\calD_Y - \sum_{k = 1}^p \theta_k \overline\calD_{X_k},\overline\calD_Y - \sum_{k = 1}^p \theta_k\overline\calD_{X_k} \rangle = \frac{1}{2} \; \frac{1}{n^3}\overset{n}{\underset{i,j,m=1}{\sum}}\Big( [\overline\calD_Y]_{ijm} - \sum_{k = 1}^p \theta_k [\overline\calD_{X_k}]_{ijm}\Big)^2,
    \end{equation*}
    where the expressions $[\calD_Y]_{ijm}$ and $[\calD_{X_k}]_{ijm}$ correspond to the quantities defined in (\ref{pc_etim_2}), $[\overline\calD_Y]_{ijm}, [\overline\calD_{X_k}]_{ijm}$ are defined as $[\overline\calD_Y]_{ijm} = [\calD_Y]_{ijm}/\sqrt{(n^{-3}\sum^n_{l,q,r}[\calD_Y]^2_{lqr})}$ and  $[\overline\calD_{X_k}]_{ijm} = [\calD_{X_k}]_{ijm}/\sqrt{(n^{-3}\sum^n_{l,q,r}[\calD_{X_k}]^2_{lqr})}$, and $\langle \calX,\calY \rangle = \sum_{k,l,r = 1}^{n}x_{klr}b_{klr}$ denotes the inner product of two same-sized tensors $\calX,\calY \in \mathbbR^{n \times n \times n}$, so that $\normalfont\widehat{\text{PC}}_v(\mathbf{X}_k,\mathbf{Y})^2 = n^{-3} \langle \overline\calD_{X_k},\overline\calD_Y \rangle = n^{-3}\sum_{i,l,r = 1}^{n}[\calD_{X_k}]_{ilr}[\calD_Y]_{ilr}/\sqrt{(n^{-3}\sum^n_{l,q,r}[\calD_{X_k}]^2_{lqr})(n^{-3}\sum^n_{l,q,r}[\calD_{Y}]^2_{lqr})}$.
\end{itemize}

\noindent Hereafter, we consider the SmRMR problem to perform feature screening with the Projection Correlation and HSIC measures. More precisely, our problem is the following one: the target is $Y \in \Rb$ (which can be extended to $q$-dimensional responses) and the vector of features is $(X_1,\ldots,X_p)^\top \in \Rb^p$. Equipped with the sample $\mathbf{Y}=(Y_1,\ldots,Y_n)$ and $\mathbf{X}_k = (X_{1k},\ldots,X_{nk})$, we compute the V-statistic estimator of $\text{D}(X_k,Y)$ and solve problem (\ref{penalized_mrmr}), which provides the set of selected features. The results we derive in the next section hold for both $\widehat{\text{PC}}_v(.,.)^2$ and $\text{nr-}\widehat{\text{HSIC}}_v(.,.)$; thus, unless stated otherwise, $\widehat{\text{D}}_v(.,.)$ will be both measures. To clarify the terminology, SmRMR with PC (resp. HSIC) refers to criterion (\ref{penalized_mrmr}) with $\widehat{\text{PC}}_v(.,.)^2$ (resp. $\text{nr-}\widehat{\text{HSIC}}_v(.,.)$).

\section{Asymptotic properties}\label{asymptotic_properties}

Our analysis is developed when the number of parameters $p$ potentially diverges with the sample size. Therefore, the number of unknown parameters, $p_n$, varies with the sample size: the sequence $(p_n)$ could tend to the infinity with $n$, that is $p=p_n\rightarrow \infty$ as $n\rightarrow \infty$. Hereafter, we denote by $\boldtheta_{n0}$ the ``pseudo-true coefficient''. Without loss of generality, we assume that the first $s_n$ components of $\boldtheta_{n0}$ are non-zero and the remaining components are zero. More formally, we assume $\theta_{n0,k} \neq 0$ for $k \leq s_n$ and $\theta_{n0,k} = 0$ for $k > s_n$ with $s_n$ the cardinality of the true sparse support, defined as $\Sc_n:=\{k = 1,\ldots,p_n:\,\theta_{n0,k}\neq 0 \}$: this is the sparsity assumption. For $\boldtheta_n \in \Theta_n$, we write $\boldtheta_{n1} = (\theta_{n,1},\ldots,\theta_{n,s_n})^\top$ and $\boldtheta_{n2} = (\theta_{n,s_n+1},\dots,\theta_{n,p_n})^\top$. We define the loss $\Lb_{v,n}(\cdot)$ as in (\ref{unpen_loss}) for a given association measure. The SmRMR problem we consider is 
\begin{equation}\label{stat_crit}
\widehat{\boldtheta}_n \, \in \, \underset{\boldtheta_n \in \Rb^{p_n}_+}{\argmin} \; n\,\Lb_{v,n}(\boldtheta_n) + n\overset{p_n}{\underset{k=1}{\sum}}\pp(\lambda_n,\theta_{n,k}),
\end{equation}
when such a minimizer exists. We assume the following regularity conditions.
\begin{assumption}\label{regularity_condition}
For a measure $\normalfont\text{D}(.,.)$, there exists a unique pseudo-true parameter value $\theta_{n0}$ minimizing the function $\normalfont\Lb(\boldtheta_n):=-\sum^{p_n}_{k=1} \theta_{n,k} \text{D}(X_k,Y) + \frac{1}{2} \sum^{p_n}_{k,l=1}{\sum}\theta_{n,k}\theta_{n,l} \text{D}(X_k,X_l)$ on $\Rb^{p_n}_+$, and an open neighbourhood of $\boldtheta_{n0}$ is contained in $\Rb^{p_n}_+$. Furthermore, for a given association measure $\normalfont\text{D}(.,.)$, any pseudo-true parameter value $\boldtheta_{n0}$ satisfies the condition, $\forall k = 1,\ldots,p_n$: $\normalfont-\text{D}(X_k,Y)+\sum^{p_n}_{l=1}\text{D}(X_k,X_l)\boldtheta_{n0,l}=0$.
\end{assumption}
\begin{assumption}\label{eigenvalue_assumption}
Let $\normalfont\text{D}_{\mathbf{X}\mathbf{X}} := (\text{D}(X_k,X_l))_{1\leq k,l \leq p_n} \in \Rb^{p_n \times p_n}$.  $\exists\underline{\mu},\overline{\mu}$ such that $\normalfont 0 < \underline{\mu} < \lambda_{\min}\big(\text{D}_{\mathbf{X}\mathbf{X}}\big) \leq \lambda_{\max}(\text{D}_{\mathbf{X}\mathbf{X}}) < \overline{\mu} < \infty$. 
\end{assumption}
Assumption \ref{regularity_condition} concerns the orthogonality condition satisfied by $\Lb(\boldtheta_{n0})$ at the true pseudo-parameter $\boldtheta_{n0}$, that is $\partial_{\theta_{n,k}}\Lb(\boldtheta_{n0})=0$, which is the derivative of the population level non-penalized loss. The parameter $\theta_{n0}$ may belong to the boundary of the parameter set in our analysis: this is a problem discussed in~\cite{poignard2024_copula}; see their Assumption 0 and the discussion relating to this point. Therefore, we define the partial derivatives w.r.t. $\theta_{n,k}$ as in the latter work when the parameter is on the boundary. Assumption \ref{eigenvalue_assumption} bounds the eigenvalues of $\text{D}_{\mathbf{X}\mathbf{X}}$, a useful condition that will facilitate the proof for consistency. As for the penalty function $\pp(\lambda_n,.)$, denoting by $\partial_2\pp(\lambda_n,x)$ (resp. $\partial^2_{2,2}\pp(\lambda_n,x)$) the first order (resp. second order) derivative of $x\mapsto \pp(\lambda_n,x)$, we assume the following conditions.
\begin{assumption}\label{assumption_penalty}
Let
$a_n := \underset{1\leq k \leq p_n}{\max}\big\{ \partial_2  \pp(\lambda_n,\theta_{n0,k}), \theta_{n0,k}\neq 0\big\}\;\text{and}
\;b_n := \underset{1\leq k \leq p_n}{\max}\big\{ \partial^2_{2,2} \pp(\lambda_n,\theta_{n0,k}),$ $\theta_{n0,k}\neq 0\big\}$,
assume that $a_n = O(n^{-1/2})$, $b_n\rightarrow 0$ as $n\rightarrow \infty$. Moreover, $\exists K_1,K_2$ such that $| \partial^2_{2,2} \pp(\lambda_n,\theta_1) - \partial^2_{2,2} \pp(\lambda_n,\theta_2) |\leq K_2 |\theta_1 -\theta_2 |$,
for any real numbers $\theta_1,\theta_2$ such that $\theta_1,\theta_2 > K_1\lambda_n  $.
\end{assumption}
\begin{assumption}\label{beta_min_assumption}
Let $\boldtheta_{n01}=(\theta_{n0,1},\ldots,\theta_{n0,s_n})^\top$ be the sub-vector of non-zero coefficients and $\boldtheta_{n02}=(\theta_{n0,s_n+1},\ldots,\theta_{n0,p_n})^\top$ be the sub-vector of zero entries, so that $\boldtheta_{n0}=(\boldtheta^\top_{n01},\boldtheta^\top_{n02})^\top$. Then $\underset{1 \leq j \leq s_n}{\min} \theta_{n0,j} / \lambda_n \rightarrow \infty, \; \text{as} \; n \rightarrow \infty$.
\end{assumption}

Assumption \ref{assumption_penalty} concerns the regularity of the penalty function satisfied by SCAD and MCP and are in the same vein as in \cite{fan2004}. For the LASSO, $a_n = \lambda_n, b_n = 0$. Assumption \ref{beta_min_assumption} is a condition on the minimum true signal and is standard in the literature on sparse M-estimation. Our first result establishes the existence of a consistent penalized M-estimator for the SmRMR problem. The proof is detailed in \ref{proofs}.
\begin{theorem}\label{bound_proba}
Suppose Assumptions \ref{regularity_condition}-\ref{assumption_penalty} are satisfied. Assume $p^{2}_ns_n\log(p_n)=o(n)$. Then, there exists a sequence $\widehat{\boldtheta}_n$ as defined in (\ref{stat_crit}) that satisfies: $\|\widehat{\boldtheta}_n-\boldtheta_{n0}\|_2 = O_p\Big(\sqrt{p_ns^2_n \log(p_n)}\Big(n^{-1/2} + a_n\Big)\Big)$.
\end{theorem}
In the case of the LASSO penalty, $a_n = \lambda_n$. This consistency result includes the original HSIC-LASSO when the association measure is HSIC with Gaussian kernel. 

The next result shows that the true zero coefficients can be recovered with probability one for the SCAD and MCP penalty functions, which ensures the correct identification of the non-relevant features.
\begin{theorem}\label{sparsistency}
Let $\pp(\lambda_n,\cdot)$ be the SCAD or MCP penalty. Under the conditions of Theorem \ref{bound_proba} and Assumption \ref{beta_min_assumption}, assume that $\lim \inf_{n \rightarrow \infty}\lim\inf_{x \rightarrow 0^+} \lambda^{-1}_n \partial_2\pp(\lambda_n,x) > 0$. Moreover, assume $\lambda_n \rightarrow 0$, $\sqrt{n/(p_ns^2_n\log(p_n))}\lambda_n\rightarrow \infty$. Then with probability tending to one, the $\sqrt{n/(p_ns^2_n\log(p_n))}$-consistent local minimizer $\widehat{\boldtheta}_n$ in Theorem \ref{bound_proba} satisfies $\widehat{\boldtheta}_{n2}=0$.
\end{theorem}
The result stated in Theorem \ref{sparsistency} relates to the sparsistency property as in \cite{lam2009}: for a consistent (local) estimator, the true zero coefficients can be recovered with probability one for $n$ large enough. This slightly differs from the support recovery property $\widehat{\Sc}_{n}=\Sc_n$ with probability tending to one: the latter holds in our framework when $p_n$ is fixed, but it is not necessarily satisfied in the large-dimensional case. To obtain a support recovery result in the case of non-convex penalty functions ensuring that the estimator $\widehat{\boldtheta}_n = \widehat{\boldtheta}_{n\Sc_n}$ with probability one, with $\widehat{\boldtheta}_{n\Sc_n}$ the oracle estimator obtained from minimizing $\Lb_{v,n}$ over the true $\Sc_n$, an alternative estimation framework would be required: see e.g., \cite{loh2017} which builds upon the restricted eigenvalue condition and the primal dual witness technique.

In the LASSO case, for the estimator to be $\sqrt{n/(p_ns^2_n\log(p_n))}$-consistent, the regularization parameter $\lambda_n$ should satisfy $\lambda_n = O(\sqrt{s^2_n\log(p_n)/n})$. But it cannot simultaneously satisfy $\sqrt{n/(p_ns^2_n\log(p_n))}\lambda_n\rightarrow \infty$. One way to fix this issue would be the adaptive LASSO, which requires the introduction of stochastic weights that depend on a first step consistent estimator. In light of the multi-stage procedure we propose hereafter in the context of ultra high-dimensional data, we prefer to avoid an additional layer of complexity that would be entailed by the adaptive LASSO.

\section{FDR control}\label{mRMR_fdr}

The issue we now investigate is the estimation of the number of false discoveries in the penalized estimator $\widehat{\boldtheta}_n$. To do so, we rely on the knockoff+ filtering method developed by \cite{barber2015}. Since the seminal work of \cite{barber2015} on the knockoff procedure and its applications for FDR control, a broad range of studies has been flourishing on the extensions and applications of the knockoff method: \cite{candes2018} devised the Model-X knockoff in the context of a random design matrix; \cite{barber2019} addressed the issue of knockoff-based FDR control within the high-dimensional setting and emphasized how knockoffs can be applied to non-sparse signals; \cite{romano2020} devised a Model-X knockoff framework almost model-free with applications to unsupervised datasets; in the same vein as \cite{barber2019}, \cite{fan2020} or \cite{liu2022} considered two-step approaches for FDR control; \cite{lu2018_deep} applied the knockoff filtering to deep neural neural networks. The method can be broken down into the following three steps:
\begin{itemize}
    \item[(1)] Construct the knockoff variables $\widetilde{X}_k$ for all $k = 1,\ldots,p_n$. To do so, two methods can be performed: the equicorrelated method or the semi-definite program method. 
    \item[(2)] For each pair of original variable $X_k$ and knockoff variable $\widetilde{X}_k$, we control the FDR using the importance score $\widehat{W}_k$ defined as $\widehat{W}_k = \widehat{\theta}_{n,k}-\widetilde{\theta}_{n,k}$, $k = 1, \ldots,p_n$, where $\widehat{\boldtheta}_n = (\widehat{\theta}_{n,k})_{1 \leq k \leq p_n}$ and $\widetilde{\boldtheta}_n= (\widetilde{\theta}_{n,k})_{1 \leq k \leq p_n}$ satisfying
    \begin{equation} \label{theta_norm_knockoff}(\widehat\boldtheta^\top_n,\widetilde\boldtheta^\top_n)^\top \in \underset{\boldtheta_n \in \Rb^{2p_n}_+}{\argmin} \; n\Gb_{v,n}(\boldtheta_n)+n\overset{2p_n}{\underset{k=1}{\sum}}\pp(\lambda_n,\theta_{n,k}), \; \text{with}
    \end{equation}
    {\small{\begin{eqnarray*}
    \begin{split}
    \Gb_{v,n}(\boldtheta_n)  = -\overset{p_n}{\underset{k=1}{\sum}} \theta_{n,k} \widehat{\text{D}}_v(\mathbf{X}_k,\mathbf{Y})- \overset{2p_n}{\underset{k=p_n+1}{\sum}}\theta_{n,k} \widehat{\text{D}}_v(\widetilde{\mathbf{X}}_{k-p_n},\mathbf{Y}) + \frac{1}{2} \overset{p_n}{\underset{k,l=1}{\sum}}\theta_{n,k}\theta_{n,l} \widehat{\text{D}}_v(\mathbf{X}_k,\mathbf{X}_l) \\ + \frac{1}{2} \overset{2p_n}{\underset{k,l=p_n+1}{\sum}}\theta_{n,k}\theta_{n,l} \widehat{\text{D}}_v(\widetilde{\mathbf{X}}_{k-p_n},\widetilde{\mathbf{X}}_{l-p_n}).
    \end{split}
    \end{eqnarray*}
    }}
    Here, a higher value of $\widehat{W}_k$ gives evidence $X_k$ is a true signal. Should $X_k$ be inactive, then $|\widehat{W}_k|$ is close to zero.
    \item[(3)] Define a selection rule to carry out feature selection. To do so, we specify a data-dependent threshold $T(\alpha)$ similar to equation (13) of \cite{barber2019}  as
    \begin{equation} \label{threshold}
    T(\alpha) = \min\big\{t \in \mathcal{W}_n: \; \frac{1+\text{card}(k:\,\widehat{W}_k\leq -t)}{\text{card}(k:\,\widehat{W}_k \geq t)\vee 1} \leq \alpha\big\},
    \end{equation}
    and $T(\alpha)=+\infty$ should this set be empty and where $\mathcal{W}_n = \big\{|\widehat{W}_k|: \; 1 \leq k\leq p_n\big\} \setminus \{0\}$ is the set of unique nonzero values reached by the $|\widehat{W}_k|$'s. Then the active set is defined as 
    \begin{equation*}
    \widehat{\Sc}_n(\alpha) = \big\{1 \leq k \leq p_n: \; \widehat{W}_k \geq T(\alpha)\big\}.
    \end{equation*}
\end{itemize}
As pointed out by \cite{barber2015}, the properties of the $W_k$ imply that 
\begin{equation*}
\text{card}\big(k \in\Sc^c_n : \; \widehat{W}_k \geq t\big) \; \overset{d}{=} \; \text{card}\big(k \in\Sc^c_n : \; \widehat{W}_k \leq -t\big) \leq \text{card}\big(k: \; \widehat{W}_k \leq -t\big),
\end{equation*}
where $\Sc_n:=\{k = 1,\ldots,p_n:\, \theta_{n0,k}\neq 0 \}$. At threshold $t$, the False Discovery Proportion (FDP) is estimated as the ratio
\begin{equation*}
\widehat{\text{FDP}}(t) = \frac{\text{card}\big(k: \; \widehat{W}_k \leq -t\big) }{\text{card}\big(k: \; \widehat{W}_k \geq t\big) \vee 1}.
\end{equation*}
One key hurdle of step (1) when constructing the knockoff variables is the sample size requirement $2p_n < n$, which is often not satisfied in practice. To circumvent this issue and apply the knockoff, we consider a data splitting procedure that builds upon Section 4.1 of \cite{barber2019}: on a first sub-sample, we perform a feature screening step to select a subset of variables, say $\widehat{S}_0 \subset \{1,\ldots,p_n\}$; using the screened variables only, we employ the knockoff procedure on the remaining data. Such data splitting procedure was also employed by \cite{liu2022} or \cite{poignard2022}. Formally, the procedure can be summarized a follows:
\begin{itemize}
    \item[(1)] Split the sample into two parts $n_0$ and $n_1$ such that $n_0 + n_1 = n$. Then define the vectors of observations $\mathbf{X}^{(0)} \in \Rb^{n_0 \times {p_n}}, \mathbf{Y}^{(0)} \in \Rb^{n_0}$ and $\mathbf{X}^{(1)} \in \Rb^{n_1 \times {p_n}}, \mathbf{Y}^{(1)} \in \Rb^{n_1}$ so that $\mathbf{Y} = (\mathbf{Y}^{(0)\top},\mathbf{Y}^{(1)\top})^\top$ and $\mathbf{X} = (\mathbf{X}^{(0)\top},\mathbf{X}^{(1)\top})^\top$.
    \item[(2)] Perform a variable screening step on the sub-sample $n_0$ according to the procedure detailed in Subsection \ref{prac:screening}: this step can be performed by solving (\ref{stat_crit}) for a given penalty, or by applying a feature screening technique such as $\text{HSIC}$ or $\text{PC}^2$, to return a set of features likely to contain the active features and together with potentially many inactive features. This set $\widehat{\Sc}_{0,n}$ is selected such that its $s_{0,n}$ variables satisfies $2 s_{0,n} < n_1$. 
    \item[(3)] Construct the knockoff variables from the subset $\mathbf{X}^{(1)}_{\widehat{\Sc}_{0,n}}$.
    Then for each $k = 1,\ldots,s_{0,n}$, with $s_{0,n} = \text{card}(\widehat{\Sc}_{0,n})$, compute the statistic for each pair of original variable and knockoff variable $\widehat{W}_k = \widehat{\theta}^{\widehat{\Sc}_{0,n}}_{n_1,k}-\widetilde{\theta}^{\widehat{\Sc}_{0,n}}_{n_1,k}$, where $(\widehat{\theta}^{\widehat{\Sc}_{0,n}}_{n_1,k}, \widetilde{\theta}^{\widehat{\Sc}_{0,n}}_{n_1,k})$ are the estimators deduced from (\ref{theta_norm_knockoff}) based on $(\mathbf{Y}^{(1)},(\mathbf{X}^{(1)}_{\widehat{\Sc}_{0,n}}, \widetilde{\mathbf{X}}^{(1)}_{\widehat{\Sc}_{0,n}}))$. The number of parameters to be estimated is $s_{0,n}$ and both estimators $\widehat{\theta}^{\widehat{\Sc}_{0,n}}_{n_1,k}$ and $\widetilde{\theta}^{\widehat{\Sc}_{0,n}}_{n_1,k}$ are non-negative. This stage corresponds to the selection step in \cite{barber2019}. Then for a fixed $\alpha$, we use (\ref{threshold}) to estimate the set of active features as
    \begin{equation*}
    \widehat{\Sc}_n(\alpha) = \big\{k: k \in \widehat{\Sc}_{0,n}, \; \widehat{W}_k \geq T(\alpha)\big\}.
    \end{equation*}
\end{itemize}
In step (2), the penalty employed for screening is the same as in step (3) and the resulting set $\widehat{\Sc}_{0,n}$ depends on the first portion of the data. A pre-screening procedure is employed in the case of ultra high-dimensional data: this aims to reduce the dimension prior to solving problem (\ref{stat_crit}). In step (3), the specification of the importance score $\widehat{W}_k$ determines the success of the FDR procedure. It should satisfy the sufficiency and anti-symmetry properties, so that the statistic has a symmetric distribution for an inactive variable and is positive or negative with the same probability. The statistic $\widehat{W}_k$ satisfies these two properties. Indeed, the penalized problem is equivalent to the minimization of
\begin{equation*}
 - n \overset{p_n}{\underset{k=1}{\sum}} \theta_{n,k} \widehat{\text{D}}_v(\mathbf{X}_k,\mathbf{Y}) + \frac{n}{2} \overset{p_n}{\underset{k,l=1}{\sum}} \theta_{n,k}\theta_{n,l} \widehat{\text{D}}_v(\mathbf{X}_k,\mathbf{X}_l)+n \overset{p_n}{\underset{k=1}{\sum}}\pp(\lambda_n,\theta_{n,k}).
\end{equation*}
Take $\widehat{\text{D}}_v(\cdot,\cdot) = \text{nr-}\widehat{\text{HSIC}}_v(\cdot,\cdot)$. Then, up to a factor of $n$, the non-penalized part of the previous problem can be expressed as follows
\begin{equation*}
 \frac{1}{n^2}\overset{n}{\underset{i,j=1}{\sum}}\sum_{k = 1}^{p_n} \theta_{n,k} [\overline\calK_{X_k}]_{ij}[\overline\calL_Y]_{ij} - \frac{1}{n^2}\overset{n}{\underset{i,j=1}{\sum}}\sum_{k,l = 1}^{p_n} \theta_{n,k}\theta_{n,l} [\overline\calK_{X_k}]_{ij} [\overline\calK_{X_l}]_{ij}.
\end{equation*}
Therefore, the problem depends on $(\mathbf{X}_1,\ldots,\mathbf{X}_{p_n},\mathbf{Y})$ through $[\overline\calK_{X_k}]_{ij}[\overline\calL_Y]_{ij}, 1 \leq k \leq p_n$, and $[\overline\calK_{X_k}]_{ij} [\overline\calK_{X_l}]_{ij}, 1 \leq k,l \leq p_n$, where $\overline\calK_{X_k}, \overline\calK_{Y}$ are continuous bounded functions of $X_k,Y$, respectively. $\widehat{W}_k$ can be expressed as $\widehat{W}_k = f((\mathbf{X},\widetilde{\mathbf{X}}),\mathbf{Y})_k$ with $f: \Rb^{2p_n+1} \rightarrow \Rb^{p_n}$.

To verify the validity of the selection step with data splitting, we assume that the screening procedure of step (2) has estimated the set $\widehat{\Sc}_{0,n}$ that includes all important features. In the case of ultra high-dimensional data, prior to estimating $\widehat{\Sc}_{0,n}$, we employ a feature pre-screening as detailed in Section \ref{feature_screening_framework}: the marginal utility $\text{D}(\cdot,\cdot)$ will be the same throughout the knockoff procedure. As emphasized in Subsection 4.2 in \cite{barber2019}, we do not assume that $\widehat{\Sc}_{0,n}$ selects the true support of $\boldtheta_{n0}$. Instead, we only assume that the screening step includes the true relevant features with probability one. Then, we define the screening event $\Ec_n$ as $\Ec_n = \big\{\Sc_n \subseteq \widehat{\Sc}_{0,n}, \; 2s_{0,n}<n_1\big\}$.
Here, the screening step performed on the sub-sample $n_0$ does not necessarily need to recover the true support $\Sc_n$ exactly. Alternatively, it is only required to provide a set of features $\widehat{\Sc}_{0,n}$ likely to contain all the active features and some inactive features. Given the screening event, the following result shows that the feature-feature mRMR procedure can control the FDR of selected features for a given level $\alpha$. The proof can be found in \ref{proofs}.
\begin{theorem}\label{fdr_control}
For any level $0 < \alpha <1$, let the data-dependent threshold $\normalfont T(\alpha) = \min\big\{t \in \mathcal{W}_n: \; \frac{1+\text{card}(k: \,\widehat{W}_k\leq -t)}{\text{card}(k: \,\widehat{W}_k \geq t)\vee 1} \leq \alpha\big\}$,
with $\mathcal{W}_n = \big\{|\widehat{W}_k|: \; 1 \leq k\leq p_n\big\} \setminus \{0\}$ and the set of active features $\widehat{\Sc}_n(\alpha) = \big\{k: k \in \widehat{\Sc}_{0,n}, \; \widehat{W}_k \geq T(\alpha)\big\}$. Then
\normalfont
\begin{equation*}
\Eb[\frac{\text{card}(k: \; k \in\Sc^c_n \;\; \text{and} \;\; k \in \widehat{\Sc}_n(\alpha)}{\text{card}(k: \; k \in \widehat{\Sc}_n(\alpha)) \vee 1}|\Ec_n] \leq \alpha.
\end{equation*}
\end{theorem}

\section{Experiments}\label{applications}

% \url{https://github.com/TANEO-bio/Sparse_modeling/blob/master/SCAD_ADMM.py}
In this section, for
simplicity of notation, we denote by $p$ the number of features, and skip the index $n$ hereafter for the parameter $\boldtheta$ and set $\Sc$ (and their estimators).
Our proposed method will be named SmRMR(D, Pen) with D (resp. Pen) referring to the association measure (resp. penalty function).
We will consider the following experiments to illustrate the theoretical guarantees and assess the relevance of our method. 
\begin{itemize}
    \item We verify that the FDR is controlled for each setting described below.
    \item We present a benchmark of the methods to compare our method based on the following metrics: accuracy, mean squared error, true positive rate, false detection rate, false positive rate on the causal features and number of features selected.
    \item Finally, we apply our method to real datasets, a collection of GWAS datasets known for their low-sample high-dimensionality configuration.
\end{itemize}

We first summarize our proposed procedure for feature selection with extra details followed by the data generating processes for the synthetic data.

\subsection{Feature selection procedure}

In Figure \ref{fig:fs_proc} we summarize the main steps of our proposed method for feature selection at a level $\alpha$. 
The pipeline takes in input a data matrix $\mathbf{X}$ with the corresponding response variable $\mathbf{Y}$, a level $\alpha$ and a set of hyper-parameters $HP$. 
The set of $HP$ can include the association measures (and kernel if appropriate), penalty function with the scalar $\lambda_n$ and optimization parameters.
The general idea of the procedure is described in \cite{barber2019} where the example of the LASSO model is given.
In particular, in the ultra high-dimensional setting, the trick involves reducing the number of covariates to a reasonable size with the screening step.
This step is followed by the knockoff filter which solves, for a given $h \in HP$, problem (\ref{stat_crit}).
We then find the best set of hyper parameters and derive the knockoff statistic that enable us to control the FDR.
The reasonable size described in the screening step is solely the minimum requirements to build the knockoff features in the model-X paradigm described in \cite{candes2016panning}, in particular the minimum requirement on the data size is $n > 2p$.

\begin{figure}[h]
    \centering
    %\Large
\tikzset{every picture/.style={line width=0.75pt}} %set default line width to 0.75pt        
\resizebox{\linewidth}{!}{%<----
\begin{tikzpicture}[x=0.75pt,y=0.75pt,yscale=-1,xscale=1]
%uncomment if require: \path (0,207); %set diagram left start at 0, and has height of 207

%Straight Lines [id:da34097341127342995] 
\draw    (122.45,94.33) -- (154.1,94.43) ;
\draw [shift={(156.1,94.43)}, rotate = 180.17] [color={rgb, 255:red, 0; green, 0; blue, 0 }  ][line width=0.75]    (10.93,-3.29) .. controls (6.95,-1.4) and (3.31,-0.3) .. (0,0) .. controls (3.31,0.3) and (6.95,1.4) .. (10.93,3.29)   ;
%Shape: Rectangle [id:dp8993023627467112] 
\draw  [color={rgb, 255:red, 74; green, 144; blue, 226 }  ,draw opacity=1 ] (157.07,81.63) -- (266.1,81.63) -- (266.1,107.43) -- (157.07,107.43) -- cycle ;
%Shape: Rectangle [id:dp3318712848493096] 
\draw  [color={rgb, 255:red, 65; green, 117; blue, 5 }  ,draw opacity=1 ] (133,9) -- (505,9) -- (505,165.43) -- (133,165.43) -- cycle ;
%Straight Lines [id:da8170349089865212] 
\draw    (266.45,94.33) -- (294.1,94.71) ;
\draw [shift={(296.1,94.73)}, rotate = 180.77] [color={rgb, 255:red, 0; green, 0; blue, 0 }  ][line width=0.75]    (10.93,-3.29) .. controls (6.95,-1.4) and (3.31,-0.3) .. (0,0) .. controls (3.31,0.3) and (6.95,1.4) .. (10.93,3.29)   ;
%Curve Lines [id:da6395918385391562] 
\draw    (150.1,55.73) .. controls (137.3,57.7) and (143.89,52.88) .. (142.18,91.92) ;
\draw [shift={(142.1,93.73)}, rotate = 272.79] [color={rgb, 255:red, 0; green, 0; blue, 0 }  ][line width=0.75]    (10.93,-3.29) .. controls (6.95,-1.4) and (3.31,-0.3) .. (0,0) .. controls (3.31,0.3) and (6.95,1.4) .. (10.93,3.29)   ;
%Shape: Rectangle [id:dp4601142979033054] 
\draw  [color={rgb, 255:red, 208; green, 2; blue, 27 }  ,draw opacity=1 ] (297.07,50.73) -- (501.5,50.73) -- (501.5,158.79) -- (297.07,158.79) -- cycle ;
%Straight Lines [id:da36815994840635413] 
\draw    (596.1,64.43) -- (596.1,78.43) ;
\draw [shift={(596.1,80.43)}, rotate = 270] [color={rgb, 255:red, 0; green, 0; blue, 0 }  ][line width=0.75]    (10.93,-3.29) .. controls (6.95,-1.4) and (3.31,-0.3) .. (0,0) .. controls (3.31,0.3) and (6.95,1.4) .. (10.93,3.29)   ;
%Shape: Rectangle [id:dp2895546533982847] 
\draw  [color={rgb, 255:red, 37; green, 37; blue, 37 }  ,draw opacity=1 ] (516,9.73) -- (679.5,9.73) -- (679.5,64.43) -- (516,64.43) -- cycle ;
%Curve Lines [id:da5747644597846235] 
\draw    (502.5,96.79) .. controls (506.38,79.33) and (508.38,74.1) .. (515.34,57.32) ;
\draw [shift={(516,55.73)}, rotate = 112.55] [color={rgb, 255:red, 0; green, 0; blue, 0 }  ][line width=0.75]    (10.93,-3.29) .. controls (6.95,-1.4) and (3.31,-0.3) .. (0,0) .. controls (3.31,0.3) and (6.95,1.4) .. (10.93,3.29)   ;
%Shape: Rectangle [id:dp5258314041928452] 
\draw  [color={rgb, 255:red, 37; green, 37; blue, 37 }  ,draw opacity=1 ] (516,81.01) -- (683.5,81.01) -- (683.5,120.43) -- (516,120.43) -- cycle ;
%Shape: Rectangle [id:dp5282146586877238] 
\draw  [color={rgb, 255:red, 37; green, 37; blue, 37 }  ,draw opacity=1 ] (537,134.43) -- (659.1,134.43) -- (659.1,163.43) -- (537,163.43) -- cycle ;
%Straight Lines [id:da6396236656538771] 
\draw    (597.1,119.43) -- (597.1,133.43) ;
\draw [shift={(597.1,135.43)}, rotate = 270] [color={rgb, 255:red, 0; green, 0; blue, 0 }  ][line width=0.75]    (10.93,-3.29) .. controls (6.95,-1.4) and (3.31,-0.3) .. (0,0) .. controls (3.31,0.3) and (6.95,1.4) .. (10.93,3.29)   ;

% Text Node
\draw  [color={rgb, 255:red, 37; green, 37; blue, 37 }  ,draw opacity=1 ]  (71.5, 95.67) circle [x radius= 50.2, y radius= 31.82]   ;
\draw (37,75.17) node [anchor=north west][inner sep=0.75pt]   [align=left] {Input data\\$\displaystyle \mathbf{X} \in \mathbb{R}^{n\times p}$};
% Text Node
\draw (135,12) node [anchor=north west][inner sep=0.75pt]   [align=left] {\textcolor[rgb]{0.25,0.46,0.02}{Cross-validation on the set of }\textcolor[rgb]{0.25,0.46,0.02}{$\displaystyle HP \ $}\textcolor[rgb]{0.25,0.46,0.02}{ by minimising }\textcolor[rgb]{0.25,0.46,0.02}{$\displaystyle \mathbb{G}_{v,n}$}};
% Text Node
\draw (159.07,84.63) node [anchor=north west][inner sep=0.75pt]  [color={rgb, 255:red, 74; green, 144; blue, 226 }  ,opacity=1 ] [align=left] {Screening step};
% Text Node
\draw (152,45) node [anchor=north west][inner sep=0.75pt]   [align=left] {Iterate over $\displaystyle h\in HP$};
% Text Node
\draw (299.07,52.73) node [anchor=north west][inner sep=0.75pt]  [color={rgb, 255:red, 74; green, 144; blue, 226 }  ,opacity=1 ] [align=left] {\textcolor[rgb]{0.82,0.01,0.11}{- Build knockoff features $\displaystyle \widetilde{\mathbf{X}}$\vspace{0.5cm}} \\ \textcolor[rgb]{0.82,0.01,0.11}{- Solve (\ref{theta_norm_knockoff}) and} \\ \textcolor[rgb]{0.82,0.01,0.11}{return $\displaystyle \theta_h = \left(\widehat{\theta}^{\widehat{\Sc}_0\top},\widetilde{\theta}^{\widehat{\Sc}_0\top}\right)^\top $}\\ \textcolor[rgb]{0.82,0.01,0.11}{- Compute $\displaystyle \mathbb{G}_{v,n}(\theta_h)$}};
% Text Node
\draw (527,25) node [anchor=north west][inner sep=0.75pt]   [align=left] {$\displaystyle h^* = \underset{h\ \in \ HP}{\text{argmin}} \ \mathbb{G}_{v,n}(\theta_h)$};
% Text Node
\draw (577,137.58) node [anchor=north west][inner sep=0.75pt]   [align=left] {$\displaystyle \widehat{\Ac}( \alpha )$};
% Text Node
\draw (543,84.4) node [anchor=north west][inner sep=0.75pt]    {$\widehat{W}_{j} =\widehat{\theta} _{h^*,j}^{\widehat{\Sc}_0} - \widetilde{\theta} _{h^*,j}^{\widehat{\Sc}_0} $};
% Text Node
\draw (273,101.73) node [anchor=north west][inner sep=0.75pt]  [font=\footnotesize] [align=left] {\begin{minipage}[lt]{13.09pt}\setlength\topsep{0pt}
\begin{center}
$\displaystyle \widehat{\mathcal{S}_{0}}$
\end{center}

\end{minipage}};

\end{tikzpicture}

}
    \caption{Overview of the proposed method}
    \label{fig:fs_proc}
\end{figure}
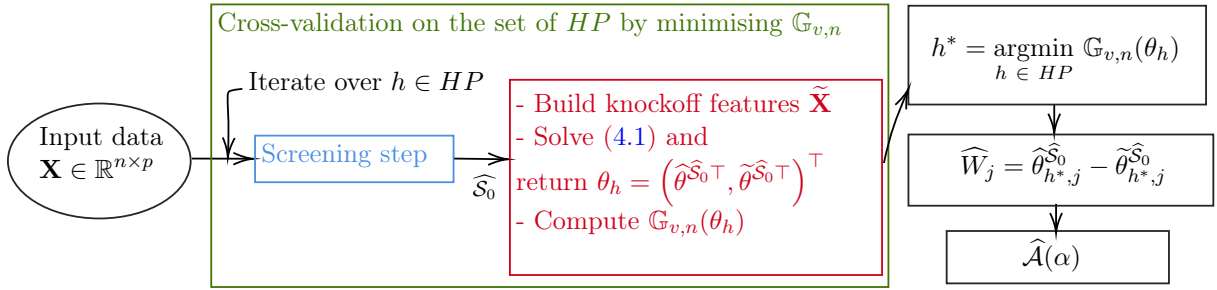

\subsubsection{Screening step} \label{prac:screening}

The screening step is essential in the high-dimensional setting.
Indeed, if $n < 2p$, we can not build the knockoff features that lead to the knockoff filter essential for controlling the FDR.
Therefore, in the same vein as in \cite{barber2019}, we split the data in two where a fraction $n_0$ of the data is used to reduce the set of features to the set $\widehat{\Sc}_{0}$ and the fraction $n_1$ is used to build the knockoff features, we have $n_0+n_1 =n$ .
Similarly, we first reduce this set by applying a marginal screening to return $\widehat{\Sc}_{0,b}$, where the cardinality of $\widehat{\Sc}_{0,b}$ is 4 times the size of $\widehat{\Sc}_{0}$.
We then further reduce its size by solving (\ref{stat_crit}) and picking the $\text{card}(\widehat{\Sc}_{0})$-top features. 
We use this method similarly to the LARS procedure described in \cite{efron2004least}, as fine tuning the LARS penalty only leads to modifying the number of retained covariates.
In this setting Data Recycling will be used and this indication is given later in the pipeline via the boolean variable \textit{DR} which is set to True.
In Figure \ref{fig:s_proc} we show a panel summarizing this step.

In addition, to increase the statistical power for the knockoff filter at no extra cost, we implement the data recycling approach introduced in \cite{barber2019} when we split the data.

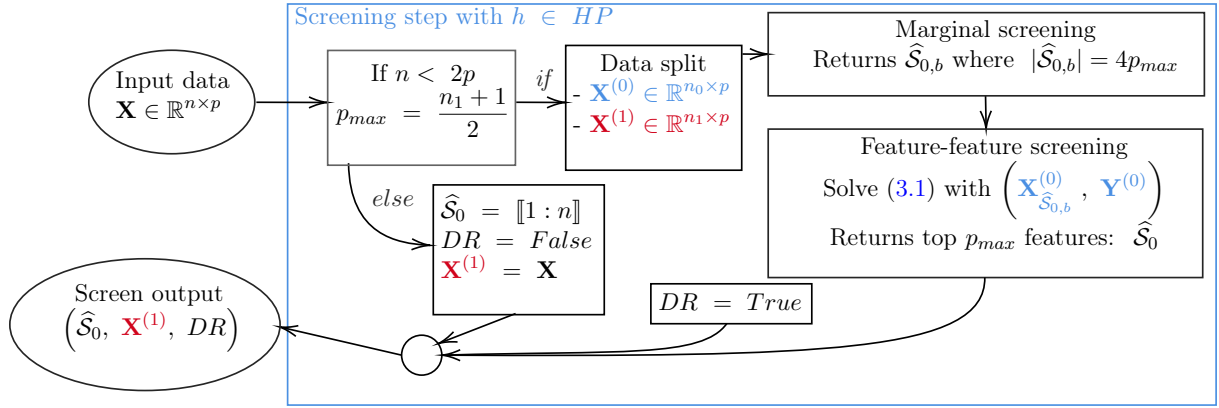
\begin{figure}[h]
    \centering
    %\Large
\tikzset{every picture/.style={line width=0.75pt}} %set default line width to 0.75pt        
\resizebox{\linewidth}{!}{%<----
\begin{tikzpicture}[x=0.75pt,y=0.75pt,yscale=-1,xscale=1]
%uncomment if require: \path (0,298); %set diagram left start at 0, and has height of 298

%Shape: Rectangle [id:dp5454510836136399] 
\draw  [color={rgb, 255:red, 74; green, 144; blue, 226 }  ,draw opacity=1 ] (173.1,25.25) -- (727,25.25) -- (727,266.35) -- (173.1,266.35) -- cycle ;
%Straight Lines [id:da6314815586154183] 
\draw    (154.45,84.33) -- (194.8,84.44) ;
\draw [shift={(196.8,84.45)}, rotate = 180.16] [color={rgb, 255:red, 0; green, 0; blue, 0 }  ][line width=0.75]    (10.93,-3.29) .. controls (6.95,-1.4) and (3.31,-0.3) .. (0,0) .. controls (3.31,0.3) and (6.95,1.4) .. (10.93,3.29)   ;
%Straight Lines [id:da40863576529798007] 
\draw    (309.45,84.33) -- (336.8,84.44) ;
\draw [shift={(338.8,84.45)}, rotate = 180.23] [color={rgb, 255:red, 0; green, 0; blue, 0 }  ][line width=0.75]    (10.93,-3.29) .. controls (6.95,-1.4) and (3.31,-0.3) .. (0,0) .. controls (3.31,0.3) and (6.95,1.4) .. (10.93,3.29)   ;
%Curve Lines [id:da5940677010370145] 
\draw    (209.45,120.93) .. controls (210.08,164.76) and (238.97,169.09) .. (256.27,170.24) ;
\draw [shift={(258.1,170.35)}, rotate = 183.37] [color={rgb, 255:red, 0; green, 0; blue, 0 }  ][line width=0.75]    (10.93,-3.29) .. controls (6.95,-1.4) and (3.31,-0.3) .. (0,0) .. controls (3.31,0.3) and (6.95,1.4) .. (10.93,3.29)   ;
%Straight Lines [id:da01570561538583659] 
\draw    (444.45,55.33) -- (457.8,55.43) ;
\draw [shift={(459.8,55.45)}, rotate = 180.44] [color={rgb, 255:red, 0; green, 0; blue, 0 }  ][line width=0.75]    (10.93,-3.29) .. controls (6.95,-1.4) and (3.31,-0.3) .. (0,0) .. controls (3.31,0.3) and (6.95,1.4) .. (10.93,3.29)   ;
%Straight Lines [id:da29210324784332853] 
\draw    (589.45,80.33) -- (589.77,98.45) ;
\draw [shift={(589.8,100.45)}, rotate = 269] [color={rgb, 255:red, 0; green, 0; blue, 0 }  ][line width=0.75]    (10.93,-3.29) .. controls (6.95,-1.4) and (3.31,-0.3) .. (0,0) .. controls (3.31,0.3) and (6.95,1.4) .. (10.93,3.29)   ;
%Shape: Rectangle [id:dp36604082460363074] 
\draw  [color={rgb, 255:red, 37; green, 37; blue, 37 }  ,draw opacity=1 ] (459.8,30.45) -- (720,30.45) -- (720,80.45) -- (459.8,80.45) -- cycle ;
%Shape: Rectangle [id:dp031361439504354216] 
\draw  [color={rgb, 255:red, 37; green, 37; blue, 37 }  ,draw opacity=1 ] (459.8,100.45) -- (719,100.45) -- (719,189.7) -- (459.8,189.7) -- cycle ;
%Straight Lines [id:da9082632432720141] 
\draw    (309.1,211.35) -- (264.96,228.62) ;
\draw [shift={(263.1,229.35)}, rotate = 338.63] [color={rgb, 255:red, 0; green, 0; blue, 0 }  ][line width=0.75]    (10.93,-3.29) .. controls (6.95,-1.4) and (3.31,-0.3) .. (0,0) .. controls (3.31,0.3) and (6.95,1.4) .. (10.93,3.29)   ;
%Shape: Circle [id:dp9633651051804901] 
\draw   (241.1,236.35) .. controls (241.1,229.72) and (246.47,224.35) .. (253.1,224.35) .. controls (259.73,224.35) and (265.1,229.72) .. (265.1,236.35) .. controls (265.1,242.98) and (259.73,248.35) .. (253.1,248.35) .. controls (246.47,248.35) and (241.1,242.98) .. (241.1,236.35) -- cycle ;
%Straight Lines [id:da7721659494425008] 
\draw    (241.1,236.35) -- (170.05,219.8) ;
\draw [shift={(168.1,219.35)}, rotate = 13.11] [color={rgb, 255:red, 0; green, 0; blue, 0 }  ][line width=0.75]    (10.93,-3.29) .. controls (6.95,-1.4) and (3.31,-0.3) .. (0,0) .. controls (3.31,0.3) and (6.95,1.4) .. (10.93,3.29)   ;
%Curve Lines [id:da24653435065536455] 
\draw    (589,189.7) .. controls (589.65,235.12) and (492.1,237.35) .. (265.1,236.35) ;
\draw [shift={(265.1,236.35)}, rotate = 0.25] [color={rgb, 255:red, 0; green, 0; blue, 0 }  ][line width=0.75]    (10.93,-3.29) .. controls (6.95,-1.4) and (3.31,-0.3) .. (0,0) .. controls (3.31,0.3) and (6.95,1.4) .. (10.93,3.29)   ;
%Curve Lines [id:da17703243812804137] 
\draw    (435.1,217.35) .. controls (420.1,238.35) and (371.1,233.35) .. (265.1,236.35) ;

% Text Node
\draw (176,25.5) node [anchor=north west][inner sep=0.75pt]   [align=left] {\textcolor[rgb]{0.29,0.56,0.89}{Screening step with }\textcolor[rgb]{0.29,0.56,0.89}{$\displaystyle h\ \in \ HP$}};
% Text Node
\draw (476.8,33.45) node [anchor=north west][inner sep=0.75pt]   [align=left] {\begin{minipage}[lt]{175.86pt}\setlength\topsep{0pt}
\begin{center}
 Marginal screening\\Returns $\displaystyle \widehat{\mathcal{S}}_{0,b}$ where \ $\displaystyle |\widehat{\mathcal{S}}_{0,b} |=4p_{max}$ 
\end{center}

\end{minipage}};
% Text Node
\draw (491,104) node [anchor=north west][inner sep=0.75pt]   [align=left] {\begin{minipage}[lt]{153.3pt}\setlength\topsep{0pt}
\begin{center}
Feature-feature screening\\Solve (\ref{stat_crit}) with $\displaystyle \left(\textcolor[rgb]{0.29,0.56,0.89}{\mathbf{X}_{\widehat{\mathcal{S}}_{0,b}}^{(0)}} \ ,\ \textcolor[rgb]{0.29,0.56,0.89}{\mathbf{Y}^{(0)}}\right)$\\Returns top $\displaystyle p_{max}$ features: \ $\displaystyle \widehat{\mathcal{S}}_{0}$
\end{center}

\end{minipage}};
% Text Node
\draw    (339,50.17) -- (444,50.17) -- (444,125.17) -- (339,125.17) -- cycle  ;
\draw (342,54.17) node [anchor=north west][inner sep=0.75pt]   [align=left] {\begin{minipage}[lt]{71.59pt}\setlength\topsep{0pt}
\begin{center}
Data split\\\mbox{-} $\displaystyle \textcolor[rgb]{0.29,0.56,0.89}{\mathbf{X}^{(0)} \in \mathbb{R}^{n_{0} \times p}}$ \\\mbox{-} $\displaystyle \textcolor[rgb]{0.82,0.01,0.11}{\mathbf{X}^{(1)} \in \mathbb{R}^{n_{1} \times p}}$ 
\end{center}

\end{minipage}};
% Text Node
\draw  [color={rgb, 255:red, 99; green, 99; blue, 99 }  ,draw opacity=1 ]  (197,53.17) -- (309,53.17) -- (309,122.17) -- (197,122.17) -- cycle  ;
\draw (200,60.17) node [anchor=north west][inner sep=0.75pt]   [align=left] {\begin{minipage}[lt]{78.53pt}\setlength\topsep{0pt}
\begin{center}
If $\displaystyle n< \ 2p$\\ $\displaystyle p_{max} \ =\ \frac{n_1+1}{2}$
\end{center}

\end{minipage}};
% Text Node
\draw (319,63.17) node [anchor=north west][inner sep=0.75pt]  [color={rgb, 255:red, 37; green, 37; blue, 37 }  ,opacity=1 ] [align=left] {\textit{if}};
% Text Node
\draw  [color={rgb, 255:red, 37; green, 37; blue, 37 }  ,draw opacity=1 ]  (104, 84.17) circle [x radius= 49.5, y radius= 31.11]   ;
\draw (69,64.17) node [anchor=north west][inner sep=0.75pt]   [align=left] {Input data\\$\displaystyle \mathbf{X} \in \mathbb{R}^{n\times p}$};
% Text Node
\draw    (389.8,193.45) -- (485.8,193.45) -- (485.8,217.45) -- (389.8,217.45) -- cycle  ;
\draw (392.8,197.85) node [anchor=north west][inner sep=0.75pt]    {$DR\ =\ True$};
% Text Node
\draw  [color={rgb, 255:red, 37; green, 37; blue, 37 }  ,draw opacity=1 ]  (88, 218) circle [x radius= 80.61, y radius= 39.6]   ;
\draw (37,192) node [anchor=north west][inner sep=0.75pt]   [align=left] {\begin{minipage}[lt]{77.82pt}\setlength\topsep{0pt}
\begin{center}
Screen output\\$\displaystyle \left(\widehat{\mathcal{S}}_{0} ,\ \textcolor[rgb]{0.82,0.01,0.11}{\mathbf{X}^{(1)}} ,\ DR\right)$
\end{center}

\end{minipage}};
% Text Node
\draw    (260,134) -- (362,134) -- (362,212) -- (260,212) -- cycle  ;
\draw (263,138) node [anchor=north west][inner sep=0.75pt]   [align=left] {$\displaystyle \widehat{\mathcal{S}}_{0} \ =\ \llbracket 1:n\rrbracket $\\$\displaystyle DR\ =\ False$\\$\displaystyle \textcolor[rgb]{0.82,0.01,0.11}{\mathbf{X}^{(1)}} \ =\ \mathbf{X}$};
% Text Node
\draw (222,135.17) node [anchor=north west][inner sep=0.75pt]  [color={rgb, 255:red, 37; green, 37; blue, 37 }  ,opacity=1 ] [align=left] {$\displaystyle else$};

\end{tikzpicture}

}
    \caption{Screening step in details, \textit{DR} is a boolean variable indicating if \textit{Data Recycling} will be used.}
    \label{fig:s_proc}
\end{figure}

\subsubsection{Knockoff filter} \label{sec:knock}

After applying the screening step to the high-dimensional data, we have $\mathbf{X}^{(1)} \in \mathbbR^{n_1\times p}$ with $n_1 > 2p$ and we can apply the knockoff filter on $\mathbf{X}^{(1)}$.
We show a summarizing panel of this step in Figure \ref{fig:kf_proc}.

\begin{figure}[h]
    \centering
    %\Large
\tikzset{every picture/.style={line width=0.75pt}} %set default line width to 0.75pt        
\resizebox{\linewidth}{!}{%<----

\begin{tikzpicture}[x=0.75pt,y=0.75pt,yscale=-1,xscale=1]
%uncomment if require: \path (0,298); %set diagram left start at 0, and has height of 298

%Shape: Rectangle [id:dp15146918701790257] 
\draw  [color={rgb, 255:red, 208; green, 2; blue, 27 }  ,draw opacity=1 ] (173.1,9) -- (616.33,9) -- (616.33,282.25) -- (173.1,282.25) -- cycle ;
%Shape: Rectangle [id:dp4952100109536256] 
\draw   (183,32.42) -- (354.1,32.42) -- (354.1,147.65) -- (183,147.65) -- cycle ;
%Shape: Rectangle [id:dp7693689296872662] 
\draw   (401,14.42) -- (608.33,14.42) -- (608.33,146.65) -- (401,146.65) -- cycle ;
%Shape: Rectangle [id:dp16659868641203934] 
\draw   (404,184.42) -- (614.33,184.42) -- (614.33,276.42) -- (404,276.42) -- cycle ;
%Shape: Rectangle [id:dp8007595728779457] 
\draw   (183,185.65) -- (354.1,185.65) -- (354.1,267.9) -- (183,267.9) -- cycle ;
%Straight Lines [id:da9003969753722135] 
\draw    (160.1,81.9) -- (180.1,81.9) ;
\draw [shift={(182.1,81.9)}, rotate = 180] [color={rgb, 255:red, 0; green, 0; blue, 0 }  ][line width=0.75]    (10.93,-3.29) .. controls (6.95,-1.4) and (3.31,-0.3) .. (0,0) .. controls (3.31,0.3) and (6.95,1.4) .. (10.93,3.29)   ;
%Straight Lines [id:da2969999902217123] 
\draw    (354.1,147.65) -- (402.41,184.44) ;
\draw [shift={(404,185.65)}, rotate = 217.29] [color={rgb, 255:red, 0; green, 0; blue, 0 }  ][line width=0.75]    (10.93,-3.29) .. controls (6.95,-1.4) and (3.31,-0.3) .. (0,0) .. controls (3.31,0.3) and (6.95,1.4) .. (10.93,3.29)   ;
%Straight Lines [id:da7566260138020819] 
\draw    (500.1,146.55) -- (500.1,183.55) ;
\draw [shift={(500.1,185.55)}, rotate = 270] [color={rgb, 255:red, 0; green, 0; blue, 0 }  ][line width=0.75]    (10.93,-3.29) .. controls (6.95,-1.4) and (3.31,-0.3) .. (0,0) .. controls (3.31,0.3) and (6.95,1.4) .. (10.93,3.29)   ;
%Straight Lines [id:da38674883877526955] 
\draw    (183.1,222.9) -- (149.1,222.57) ;
\draw [shift={(147.1,222.55)}, rotate = 0.56] [color={rgb, 255:red, 0; green, 0; blue, 0 }  ][line width=0.75]    (10.93,-3.29) .. controls (6.95,-1.4) and (3.31,-0.3) .. (0,0) .. controls (3.31,0.3) and (6.95,1.4) .. (10.93,3.29)   ;
%Straight Lines [id:da8752871249014019] 
\draw    (354.1,91.9) -- (399.1,91.9) ;
\draw [shift={(401.1,91.9)}, rotate = 180] [color={rgb, 255:red, 0; green, 0; blue, 0 }  ][line width=0.75]    (10.93,-3.29) .. controls (6.95,-1.4) and (3.31,-0.3) .. (0,0) .. controls (3.31,0.3) and (6.95,1.4) .. (10.93,3.29)   ;
%Straight Lines [id:da7686029027053436] 
\draw    (356.1,222.9) -- (404.1,222.9) ;
\draw [shift={(354.1,222.9)}, rotate = 0] [color={rgb, 255:red, 0; green, 0; blue, 0 }  ][line width=0.75]    (10.93,-3.29) .. controls (6.95,-1.4) and (3.31,-0.3) .. (0,0) .. controls (3.31,0.3) and (6.95,1.4) .. (10.93,3.29)   ;

% Text Node
\draw  [color={rgb, 255:red, 37; green, 37; blue, 37 }  ,draw opacity=1 ]  (80, 81) circle [x radius= 80.61, y radius= 39.6]   ;
\draw (28,55) node [anchor=north west][inner sep=0.75pt]   [align=left] {\begin{minipage}[lt]{77.82pt}\setlength\topsep{0pt}
\begin{center}
Input\\$\displaystyle \left(\widehat{\mathcal{S}}_{0} ,\ \mathbf{\textcolor[rgb]{0.82,0.01,0.11}{X}}\textcolor[rgb]{0.82,0.01,0.11}{^{( 1)}} ,\ DR\right)$
\end{center}

\end{minipage}};
% Text Node
\draw (185,37.17) node [anchor=north west][inner sep=0.75pt]   [align=left] { Knockoff features \ $\displaystyle \textcolor[rgb]{0.82,0.01,0.11}{\widetilde{\mathbf{X}}_{_{\widehat{\mathcal{S}}_{0}}}^{( 1)}}$\\ \ \ $\displaystyle \mathbf{X}_{_{S}} =\left(\textcolor[rgb]{0.82,0.01,0.11}{\mathbf{X}_{_{\widehat{\mathcal{S}}_{0}}}^{( 1)}},\textcolor[rgb]{0.82,0.01,0.11}{\ \widetilde{\mathbf{X}}_{_{\widehat{\mathcal{S}}_{0}}}^{( 1)}}\right)$ \\ \ \ $\displaystyle \mathbf{Y}_{s} \ =\textcolor[rgb]{0.82,0.01,0.11}{\ \mathbf{Y}^{(1)}}$};
% Text Node
\draw (405,17.17) node [anchor=north west][inner sep=0.75pt]   [align=left] { \ \ \ \ \ \ \ \ \ \ Data recycling\\ \ \ \ \ \ \ \ \ \ \ \ \ \ (Row wise)\\ \textcolor[rgb]{1.,1.,1.}{white} \\ $\displaystyle \mathbf{X}_{_{S}} =\left(\left(\textcolor[rgb]{0.29,0.56,0.89}{\mathbf{X}{_{_{\widehat{\mathcal{S}}_{0}}}^{( 0)}}^\top},\textcolor[rgb]{0.29,0.56,0.89}{\mathbf{X}{_{_{\widehat{\mathcal{S}}_{0}}}^{( 0)}}^\top}\right) ,\ \mathbf{X}_{_{S}}^\top\right)^\top$ \\ $\displaystyle \mathbf{Y}_{s} \ =\ \left(\textcolor[rgb]{0.29,0.56,0.89}{\mathbf{Y}_{0}^\top} ,\ \mathbf{Y}_{s}^\top\right)^\top$};
% Text Node
\draw (412,194.65) node [anchor=north west][inner sep=0.75pt]   [align=left] {Model fitting\\Solve (4.1) with $\displaystyle (\mathbf{X}_{S} ,\ \mathbf{Y}_{S})$\\to obtain \textcolor[rgb]{0.82,0.01,0.11}{$\displaystyle \theta _{h} =\left(\widehat{\theta }^{\widehat{\mathcal{S}}_{0} ,\top} ,\widetilde{\theta }^{\widehat{\mathcal{S}}_{0} ,\top}\right)^\top$}};
% Text Node
\draw (190,215.17) node [anchor=north west][inner sep=0.75pt]   [align=left] {\begin{minipage}[lt]{109.99pt}\setlength\topsep{0pt}
\begin{center}
Compute \ $\displaystyle \mathbb{G}_{v,n}( \theta _{h})$
\end{center}

\end{minipage}};
% Text Node
\draw (175.1,12) node [anchor=north west][inner sep=0.75pt]  [color={rgb, 255:red, 208; green, 2; blue, 27 }  ,opacity=1 ] [align=left] {Knockoff filter with $\displaystyle h\ \in \ HP$};
% Text Node
\draw (356,72) node [anchor=north west][inner sep=0.75pt]  [font=\footnotesize] [align=left] {$\displaystyle If\ DR$};
% Text Node
\draw (373,150) node [anchor=north west][inner sep=0.75pt]  [font=\footnotesize] [align=left] {$\displaystyle else$};
% Text Node
\draw  [color={rgb, 255:red, 37; green, 37; blue, 37 }  ,draw opacity=1 ]  (104, 222) circle [x radius= 43.84, y radius= 15.56]   ;
\draw (74,213) node [anchor=north west][inner sep=0.75pt]   [align=left] {\begin{minipage}[lt]{42.93pt}\setlength\topsep{0pt}
\begin{center}
$\displaystyle \mathbb{G}_{v,n}( \theta _{h})$
\end{center}

\end{minipage}};

\end{tikzpicture}

}
    \caption{Knockoff feature steps in details}
    \label{fig:kf_proc}
\end{figure}
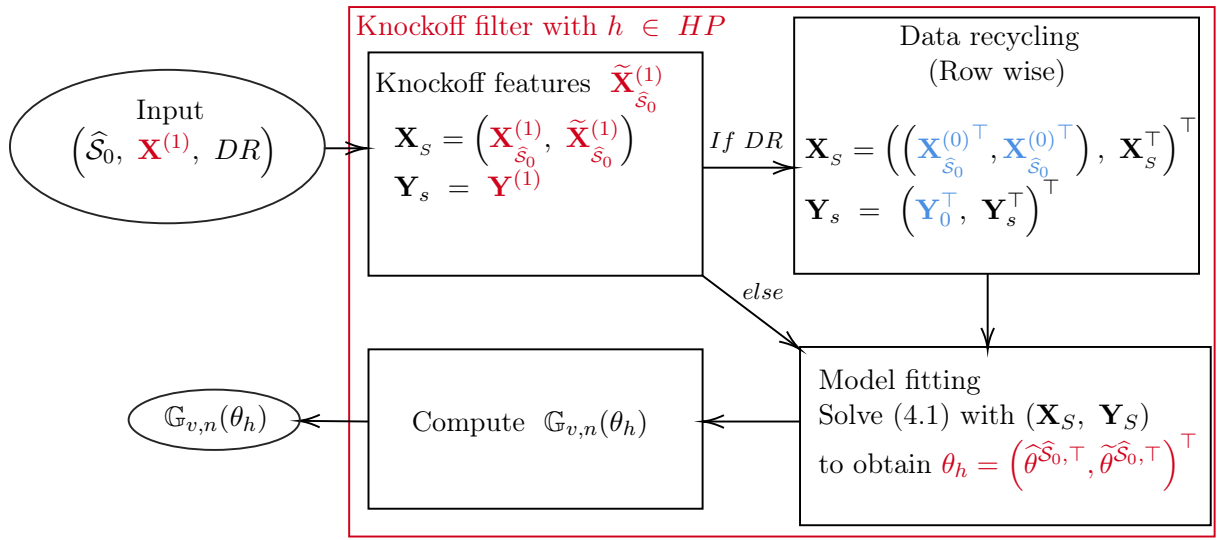

\paragraph{Solving problem (\ref{theta_norm_knockoff})} 
For convex penalties such as the LASSO, the objective function can be re-written as a Quadratic programming function and the sum of convex functions, hence it can be solved very efficiently with open-source software such as CVXPY from \cite{diamond2016cvxpy}. When the loss function is strictly non-convex such as is the case with SCAD or MCP, we have to use gradient based methods or local linear approximations (LLA), as described in \cite{zou2008one} and \cite{fan2014strong}. Solving via gradient descent involves many additional parameters and a correctly calibrated optimizer which can be quite unstable and difficult to implement. In contrast the LLA algorithm allows its implementation in a straightforward way and the re-use of already available written software.
We give the process for the LLA in Algorithm~\ref{alg:cap}. For SCAD and MCP, we initialize the algorithm by solving the LASSO penalized equation. $m$ is set to 2, following \cite{fan2014strong}. The stopping criteria defined by $||.||$ is set in our case to the $L_2$-norm.

\begin{algorithm}
\caption{Local linear approximation (LLA) algorithm}\label{alg:cap}
\begin{algorithmic}
\Require $(m, \varepsilon)$
\Return $\widehat{\theta}^{(s)}$
\State Initialize $\beta^{(0)}$ \Comment{With the LASSO penalization for example}
\State $\widehat{\theta}^{(0)} \gets \pp(\lambda_n,\widehat{\beta}^{(0)})$
\State $s \gets 1$
\While{$s \neq m$ and $|| \widehat{\theta}^{(s)} - \widehat{\theta}^{(s-1)} || < \varepsilon$} \Comment{Stopping criteria}
    \State Solve $\widehat{\beta}^{(s)} = \underset{\beta: \beta_k>0}{\min} \ \mathbb{L}_{v,n}(\beta)+\sum_k \widehat{\theta}^{(s-1)}_k \cdot|\beta_k|$
    \State $\widehat{\theta}^{(s)} \gets \pp(\lambda_n,\widehat{\beta}^{(s)})$
    \State $s \gets s+1$
\EndWhile
\end{algorithmic}
\end{algorithm}

\paragraph{Cross-validation and hyper-parameter optimization}
We wish to minimize the non-penalized part in (\ref{stat_crit}) on a left out validation set to find the best set of hyper parameters in an unbiased manner.
This procedure can be fine-tuned, in particular for solving problem (\ref{stat_crit}) in step 2  of the screening step and problem (\ref{theta_norm_knockoff}) with the knockoff features.
However, we do not perform any fine-tuning in the screening step for two reasons: we retain only the top performing features to reduce the input space dimensionality as in LARS, modifying the penalization scalar only leads to shifting the feature acceptance threshold, which leads only to more or less features (their relative importance remains the same); to reduce the computational burden.
By default we take $\lambda_n = 0.01$ for this step.
In step 3 of the knockoff filter, we apply cross-validation for fine-tuning.
In particular, when appropriate we fine-tune with respect to $\lambda_n$ via Bayes optimization methods, see \cite{DewanckerMC16}.

\paragraph{Constructing knockoff variables}
We use the second-order model-X knockoff of \cite{candes2018} described in their Section 3.2: instead of requiring $(\mathbf{X},\widetilde{\mathbf{X}})_{\text{swap}(\Sc)} \overset{d}{=} (\mathbf{X},\widetilde{\mathbf{X}})$, where  $(\mathbf{X},\widetilde{\mathbf{X}})_{\text{swap}(\Sc)}$ is obtained from $(\mathbf{X},\widetilde{\mathbf{X}})$ when swapping the entries $X_k$ and $\widetilde{X}_k$ for each $k \in \Sc$, the second-order method requires that $(\mathbf{X},\widetilde{\mathbf{X}})_{\text{swap}(\Sc)}$ and $(\mathbf{X},\widetilde{\mathbf{X}})$ to have the same mean and variance-covariance. To ensure the positive semidefinite construction of the latter variance-covariance, we employ the equicorrelated method.

\paragraph{Setting $\alpha$} We consider that a good value for $\alpha$ would be in the range of $(0.2, 0.4)$ depending on the data. 
The model can be conservative and consider that the risk is too high and return an empty set of features. 
This situation can be problematic when we fit the proposed procedure into a larger pipeline where the ultimate end-goal is to perform a classification or a regression.
In this situation only, i.e., when no features are returned for the set $\alpha$, we add a while loop in the last step of the knockoff filter to increase the value of $\alpha$ by $0.05$. If however we reach $\alpha =1$ and still no features are retained, we modify the algorithm to return the feature that maximises the knockoff statistic. This variant of SmRMR will be named SmRMR\textsubscript{2}. Unless, stated otherwise, $\alpha$ is set to a default value of $0.3$.

\paragraph{Choice of the knockoff statistic}
Several statistics are proposed in \cite{barber2015,barber2019} and satisfy the sufficiency and anti-symmetry properties. Throughout our empirical applications, we employ the statistic 
$\forall k \in \widehat{\Sc}, \ \widehat{W}_k = \widehat{\theta}_k - \tilde{\theta}_k$.

\subsection{Synthetic data} \label{s:data}

Similar to \cite{poignard2022}, we consider multiple data generating processes (DGP) with different sample sizes and relationships with the output.
We denote by $X \in \Rb^{n \times p}$ the matrix of covariates containing $n \in \Nb^*$ samples and $p \in \Nb^*$ features and by $Y\in \Rb^n$ the target output. 
Each input sample is drawn from a Gaussian distribution, $X_i \sim \mathcal{N}(0_p, \Sigma)$ where $\Sigma = (\sigma_{k,l})_{p\times p}$ and $\sigma_{k,l} = c^{|k-l|}$, where we set $c=0.5$, except stated otherwise. 
We denote by $\Sc$ the index of the true covariates.
We denote by $\varepsilon$ the error term, which is set by default as $\varepsilon \sim \mathcal{N}(0,1)$ except stated otherwise. 
% We denote by $\mathcal{T}(k)$ the Student distribution with $k$ degrees of freedom. 
We denote by $\mathbf{1}_n$ the $n$-dimensional vector of one's. Finally we denote by $\cdot$ the vector product. Equipped with these notations, we consider the following DGPs:
\paragraph{Linear models:}
\begin{itemize}
    \item[\textbf{1.a}] : $Y = \beta_0 . X_\Sc + \varepsilon, \text{ where } \beta_0 = (4, 8), \, \Sc=\lbrace 0, 5\rbrace \text{ and  } c=0$.
    \item[\textbf{1.b}] : $\text{Similar to 1.a except that }  \beta_0 = (1, 2, 4, 8) \text{ and  } \Sc=\lbrace 0, 10, 20, 30\rbrace$.
    \item[\textbf{1.c}] : $\text{Similar to 1.b except that } c=0.5$.
    \item[\textbf{1.d}] : $Y = \mathbf{1}_{10} . X_\Sc + \varepsilon, \text{ where } \Sc=\lbrace 0, 10, ..., 80, 90\rbrace$.
\end{itemize}
\paragraph{Non linear models:}
\begin{itemize}
    \item[\textbf{2.a}] : $Y=5 X_{0}+2 \sin \left(\pi X_{10} / 2\right)+2 X_{20} \mathds{1}\left\{X_{20}>0\right\}+ 2 \exp \left(5 X_{30}\right)+\varepsilon$.
    \item[\textbf{2.b}] : $Y = 3 X_{0}+3 X_{10}^{3}+3 X_{20}^{-1}+5 \mathds{1}\left\{X_{30}>0\right\}+\varepsilon$.
    \item[\textbf{2.c}] : $Y \sim \mathcal{P}(\mathbf{1}_{10}.X_\Sc), \text{ where } \mathcal{P} \text{ is the Poisson distribution}  \text{ and } \Sc=\lbrace 0, 10, ..., 80, 90\rbrace$.
\end{itemize}

\paragraph{Categorical data:}
\begin{itemize}
    \item[\textbf{3.a}] : $Y= \mathds{1}\lbrace t > 1 \rbrace \; \text{ where} \;  t= \exp\left(X_0 + X_5\right)) \text{ and  } c=0 \;$.
    \item[\textbf{3.b}] : $\text{Similar to 3.a except that } t=\exp\left(\sum_{i=0}^9 X_{10i}\right)$.
    \item[\textbf{3.c}] : $\text{Similar to 3.b except that } c=0.5$.
\end{itemize}

For each DGP, we vary the sample size $n \in \left\{100, 500 \right\}$ and the number of features $p \in \left\{100, 500, 5000 \right\}$, except for $n=500$ paired with $p=100$ and $p=500$ for computational reasons. When running the knockoff procedure, we set $n_0 = \lfloor 0.4n \rfloor$, $s_0$ is set to take the maximum amount feature while respecting $2n > p$, in particular we set $p_{max} = (n-1) / 2$.

\subsection{Benchmark performance}

We run a benchmark on all the configurations given in Subsection \ref{s:data} with the mRMR algorithm, the HSIC-LASSO implementation from \cite{climente2019block}, and our method. When we use HSIC with a Gaussian kernel, we set the width to the median heuristic. 
HSIC-LASSO and our method with a Gaussian kernel are similar but differ in two ways. Firstly, what we name as HSIC-LASSO is actually the Block HSIC-LASSO from \cite{climente2019block} where the kernel matrix is not fully computed but partitioned to scale the method to high-dimensional data. Secondly, our method has a pre-screening step to allow the creation of knockoff features. For each method, depending whether it is a classification or regression task, we retain the selected features from the previous method and apply a logistic regression (classification) or a LASSO (regression).
In this section we will show a single figure, but all the results can be found in Figures \ref*{fig:linearDGP}, \ref*{fig:linearDGP_smrmr}, \ref*{fig:FnonlinearDGP}, \ref*{fig:catDGP} and \ref*{fig:catDGP_smrmr} in Section \ref*{app:exp} of the Supplementary Material. 
In particular, we use two different association measures, PC and HSIC (with a Gaussian kernel).
In addition, we experiment by modifying the penalty in the SmRMR framework to: no penalty, L1, MCP or SCAD.

%% performance measures
\paragraph{The general trend} 
All models performance decrease with more complex DGP processes, for example, in the simplest DGP, i.e. linear 1.a (Section \ref*{app:exp} in the Supplementary Material, Figure~\ref*{fig:linearDGP}), increasing $p$ from 100 to 500 and 5,000 while keeping $n = 100$ leads to the MSE of SmRMR(PC, None) to increase from $3.8$, to $6.7$, and $9.5$ respectively.
Similarly, increasing $n$ to 500 reduces its MSE to $1.7$.
The TPR and the FDR follow a similar trend for all DGP and methods (Section \ref*{app:exp} in the Supplementary Material).
Increasing the complexity of the DGP leads to a decrease in performance. For example, for $n=100$ and $p=100$ shown in Figure~\ref*{fig:linearDGP}, going from DGP 1.a, to 1.b (adding two causal features), to 1.c (adding correlation between the variables), modifies the MSE from $3.8$ to $9.1$ and $15$ respectively.
The best performing model, HSIC-LASSO, outperforms our method in terms of accuracy or MSE. This however comes at the cost of a much larger number of selected features as we notice in every DGP, except for 1.a: this is due to the bias generated by the LASSO that prevents data from choosing
a large $\lambda_n$. This is in line with the empirical findings of \cite{fan2009}. This higher number of selected features directly affects the FDR, in addition in most cases, our method exhibits a lower FDR.

\paragraph{FDR control} It was demonstrated that the SIS property was satisfied and that FDR is controlled on average when following the knockoff+ procedure, as shown in \citeauthor{poignard2022} \cite{poignard2022}.
From Figure \ref{fig:nonlinearDGP}, we notice that the FDR is controlled for SmRMR with FDR rates below $\alpha$ but that the FDR is not controlled in every situation for SmRMR\textsubscript{2}, the modified method to prevent the empty set return.
Described in Section \ref{sec:knock}, the method was modified to return at least one feature, even in the most conservative situations.
This modification upper biases and naturally leads to an increase in the FDR. In \cite{poignard2022}, it was shown that the empirical probability of returning an empty set for the knockoff+ method can be as high as 80\% for the values of $\alpha$ set to $0.3$ as is the case here. 
This effect was also reported in \cite{liu2022}, where there put in relation the number of active features and $\alpha$ to returning the empty set and this is not ideal for assessing the predictive performance of the method.

Inspection of the TPR, FDR, and number of selected features ($N$) in Figure \ref{fig:nonlinearDGP}, show that for 2.a and 2.c, the TPR and FDR values for the PC association measure outperform HSIC-LASSO. Specially for 2.c when $n=500$ and $p=5000$ where SmRMR(PC, MCP) reaches TPR $=0.58$ and FDR $=0.32$ compared to TPR $=0.20$ and FDR $=0.32$ on average. 

%% Effect of penalty
\paragraph{The effect of the penalty function} Between no penalty, L1, SCAD and MCP penalty no clear and strong trend can be observed expect that the use of a penalty is positive and that MCP seems to perform better by a large margin on the categorical DGP (DGP defined by 3.a, 3.b and 3.c) for both association measures. 
On the contrary, for the non-linear process, we notice that MCP combined with PC is the best performing method on 2.b, and 2.c. However, we do not notice the same trend for the HSIC association measure and MCP.

%%% ADD figure
\begin{figure}[h]
\includegraphics[width=\textwidth]{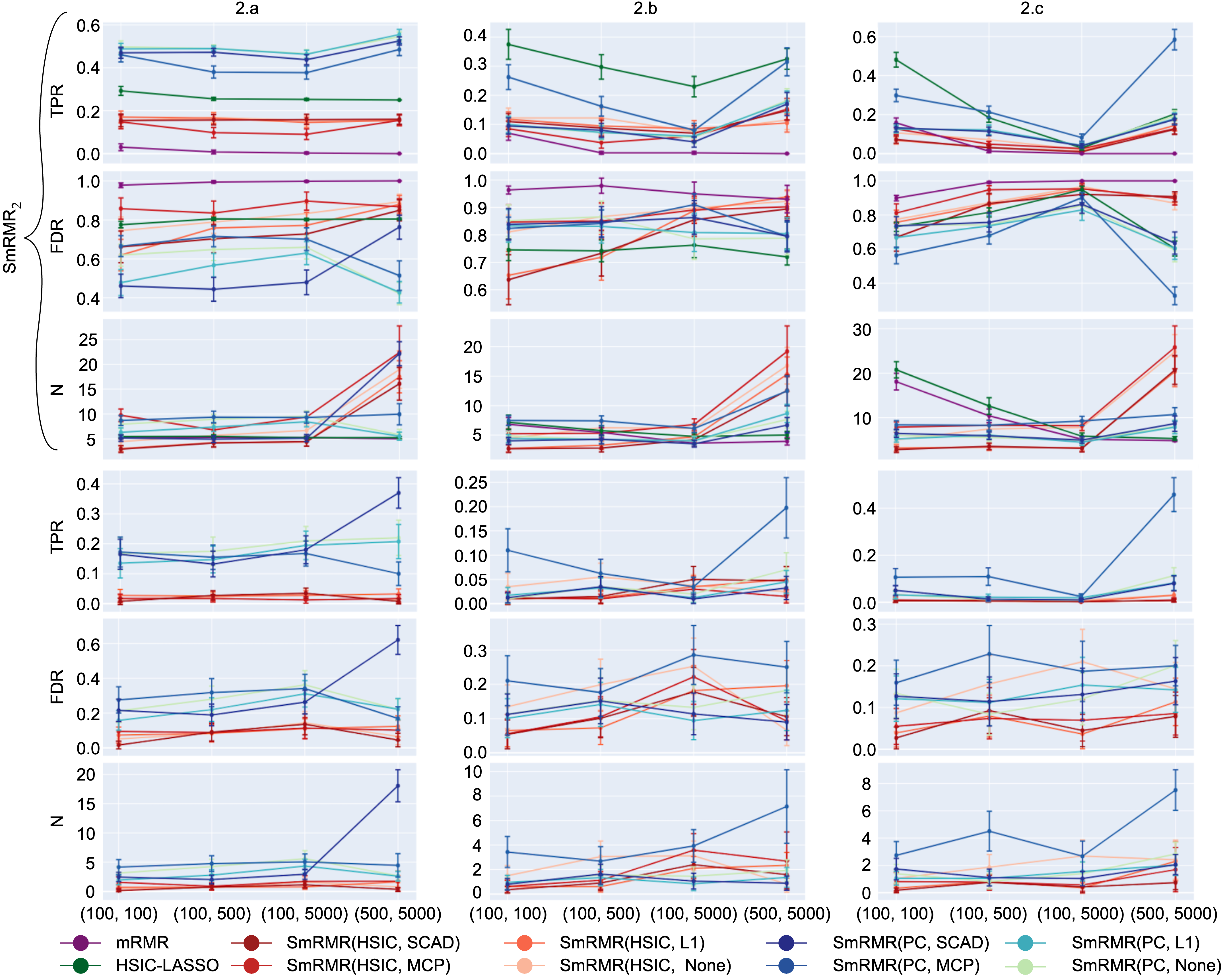}
\caption{Results for the non-linear DGP. For HSIC, we use the Gaussian kernel. $\alpha$ is set to $0.3$. In the top 3 row we use the modified version SmRMR\textsubscript{2} and for the final 3 rows we use the unmodified version.}
\label{fig:nonlinearDGP}
\end{figure}

\subsection{Real-world data}

We evaluate the proposed method on eight real-world datasets commonly used to benchmark feature selection algorithms: see Table~\ref{tab:rwd} for a description of these datasets. To select the hyper-parameters and to estimate the performance in an unbiased fashion, we use a 5 fold nested cross-validation, similar to \cite{naylor2022prediction}. In the inner cross-validation, we select the hyper-parameters ($k$ in HSIC-LASSO, $\alpha$ in SmRMR), using 64\% of the samples for training, and 16\% for validation. We choose the hyper-parameters that maximize the accuracy on the validation set. Then, we select the features on the train set using the best hyper-parameters, and train a random forest model on the full train and validation set using only the selected features. Last, we estimate the performance of the model in the outer loop, containing the remaining 20\% of the samples. We used a random forest model, instead of a logistic regression, as tree-based models are good at modeling non-linear dependencies in tabular data: see \cite{climente2019block}. We repeat 10 times this whole procedure and average the performance. We show the classification performance on outer-cross-validation test set in Figure \ref{fig:bio_scores} and the number of selected features in Figure \ref{fig:deform}.

For non-linear dependencies, we favoured the PC association measure as shown in Figure~\ref{fig:nonlinearDGP} and the use of SCAD or MCP, hence we compare SmRMR with the PC association with either SCAD or MCP penalty on the real-world datasets.

We did not include mRMR as this method was computationally expensive and it was highlighted in \cite{climente2019block} that its performance was inferior to that of HSIC-LASSO. SmRMR achieves better performance than HSIC-LASSO on both TOX\_171 and orlraws10P; on the other datasets, HSIC-LASSO outperforms SmRMR. However, as shown in Figure \ref{fig:features}, these better scores come at the cost of about 3 times more selected features.
Finally, from the real-world datasets experiments, we notice that SCAD seems to be slightly more conservative in terms of feature selection with respect to MCP.

\begin{table}
\caption{Description of the real-world datasets used for benchmark.}
\label{tab:rwd}
\begin{center}
\resizebox{\textwidth}{!}{
\begin{tabular}{ |c|c|c|c|c| } 
\hline
Name                 & Data modality                        & Number of samples & Number of features & Reference                        \\ \hline
CLL\_SUB\_111        &                                      & 111               & 11,340             &                                  \\ \cline{1-1}\cline{3-4}
GLIOMA               &                                      & 50                & 4,434              &                                  \\ \cline{1-1}\cline{3-4}
SMK\_CAN\_187        &                                      & 187               & 19,993             &                                  \\ \cline{1-1}\cline{3-4}
TOX\_171             &  \multirow{-4}{*}{Gene expression}   & 171               & 5,748              &                                  \\ \cline{1-4}
orlraws10P           &                                      & 100               & 10,304             &                                  \\ \cline{1-1}\cline{3-4}
pixraw10P            &                                      & 100               & 10,000             &                                  \\ \cline{1-1}\cline{3-4}
warpAR10P            &  \multirow{-3}{*}{Image}             & 130               & 2,400              & \multirow{-7}{*}{\cite{li2018}} \\ \hline
Toxicity             & Chemical properties                  & 171               & 1,203              & \cite{Gul2021}                  \\ \hline
\end{tabular}}
\end{center}
\end{table}

\begin{figure}[h]
\centering
\begin{subfigure}{0.44\textwidth}
\centering
    \includegraphics[width=1.0\linewidth]{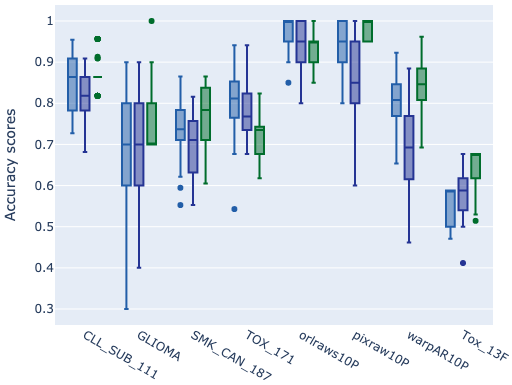}
    \caption{Accuracy scores}
    \label{fig:bio_scores}
\end{subfigure}%
\begin{subfigure}{0.56\textwidth}
\centering    \includegraphics[width=1.0\linewidth]{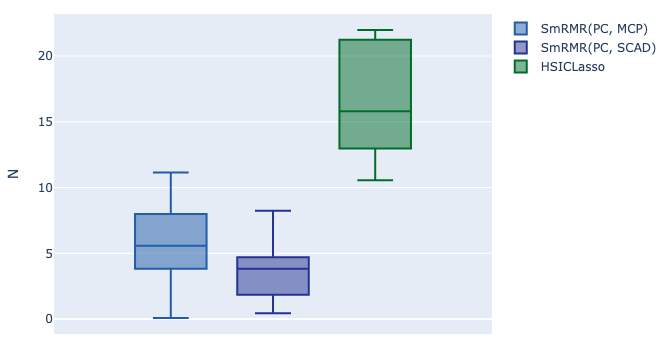}\vspace{0.8cm}
    \caption{Average number of selected features}
    \label{fig:deform}
\end{subfigure}
\caption{Real-world biological datasets application with SmRMR and HSIC-LASSO.}
\label{fig:features}
\end{figure}

\section{Conclusion}

We introduced a novel framework for feature screening that leverages both feature-feature and feature-target relationships. By employing the penalized mRMR procedure, we have successfully identified inactive features and provided theoretical conditions for accurate zero parameter recovery. 

The proposed multi-stage approach, incorporating the knockoff filter, based on our empirical evaluations on simulated and real-world datasets demonstrate the effectiveness of the proposed method in terms of performance and in terms of feature selection. While being at similar levels of performance to HSIC-LASSO, our method returns fewer parameters, thus decreasing the FDR levels. In addition, it only requires a threshold level for FDR rather than a predetermined number of features in contrast to the other approaches. This advantage simplifies the practical application of the method.

In conclusion, the proposed framework offers a valuable theoretical tool for feature screening, providing a robust and efficient approach for identifying inactive features. Its potential applications extend to various domains where feature selection is a critical step in data analysis and modeling such as bio-informatics.

\acks{Benjamin Poignard was supported by JSPS KAKENHI (22K13377) and RIKEN AIP. Héctor Climente-González was supported by the RIKEN Special Postdoctoral Researcher Program.}

{
\vskip 0.2in
\bibliography{main.bib}
%\printbibliography
}
\newpage
%\appendixhead

\appendix
    
\numberwithin{equation}{section}
% \makeatletter 
% % "activate" the preparatory code, but for section-level headers only
\renewcommand{\thesection}{Appendix \Alph{section}}
\section{Preliminary results}\label{deviation}
\renewcommand{\thesection}{\Alph{section}}

Our asymptotic results rely on the deviation inequality $|\widehat{\text{D}}_v(\mathbf{X},\mathbf{Y})-\text{D}(X,Y)| > \eps$, for $\eps>0$. In this Appendix, we present two deviation inequalities: an inequality when $\text{D}(\cdot,\cdot)$ is the Projection Correlation; an inequality when $\text{D}(\cdot,\cdot)$ is the normalized HSIC. 

\begin{lemma}\label{exponential_pc}[Theorem 1 of \cite{liu2022}]\\
For any $0 < \eps <1$, there exists a positive constant $a_1>0$ finite such that
\begin{equation*}
\normalfont\Pb\Big(\big|\widehat{\text{PC}}_v(\mathbf{X},\mathbf{Y})^2-\text{PC}(X,Y)^2\big|>\eps\Big) \leq O\Big(\exp\big(-a_1 n \eps^2\big)\Big).
\end{equation*}
\end{lemma}

\begin{lemma}\label{exponential_hsic}
Assume that $k(\cdot,\cdot)$ and $l(\cdot,\cdot)$ are non-negative and symmetric kernels bounded by $1$. Then for any $0 < \eps <1$, there exists a positive constant $b_1>0$ finite such that
\begin{equation*}
\normalfont\Pb\Big(\big|\text{nr-}\widehat{\text{HSIC}}_v(\mathbf{X},\mathbf{Y})-\text{nr-HSIC}(X,Y)\big|>\eps\Big) \leq O\Big(\exp\big(-b_1 n \eps^2\big)\Big),
\end{equation*}
where $\normalfont\text{nr-}\widehat{\text{HSIC}}_v(\mathbf{X},\mathbf{Y})$ and $\normalfont\text{nr-HSIC}(X,Y)$ are given by (\ref{nr-HSIC-vstat}) and (\ref{nr-HSIC}), respectively.
\end{lemma}
\begin{remark}
Since both exponential bounds have the same convergence rate, we do not distinguish the normalized HSIC and Projection Correlation cases in the proofs of Theorem \ref{bound_proba} and Theorem \ref{sparsistency}. Furthermore, Lemma \ref{exponential_hsic} can be stated for any symmetric bounded kernels $\|k\|_{\infty}=a_1$ and $\|l\|_{\infty}=a_2$.
\end{remark}

\begin{proof}[Proof of Lemma \ref{exponential_hsic}.]
First, we derive a deviation inequality for $|\widehat{\text{HSIC}}_v(\mathbf{X},\mathbf{Y})-\text{HSIC}(X,Y)|$. Then, using, e.g., Lemma S.2 of \cite{liu2022}, we deduce Lemma \ref{exponential_hsic}.\\
To derive the deviation inequality, it is useful to consider the decomposition of the V-statistic estimator of $\text{HSIC}(X,Y)$ and analyze its bias. This estimator can be written as a combination of three V-statistics:
\begin{eqnarray*}
\widehat{\text{HSIC}}_v(\mathbf{X},\mathbf{Y}) 
& = & \frac{1}{n^2}\overset{n}{\underset{i,j=1}{\sum}} K_{ij}L_{ij} + \frac{1}{n^4} \overset{n}{\underset{i,j,m,l=1}{\sum}} K_{ij}L_{ml} - \frac{2}{n^3} \overset{n}{\underset{i,j,m=1}{\sum}} K_{ij} L_{im} := \widehat{S}^v_1 + \widehat{S}^v_2 - 2 \widehat{S}^v_3.
\end{eqnarray*}
We decompose each V-statistic following Subsection A.1 in \cite{gretton2007kernel}. First, noting that $(n)_c/n^c = n!/((n-c)!n^c)=1+O(n^{-1})$, we have:
{\small{\begin{eqnarray*}
\begin{split}
\Eb[\widehat{S}^v_1]= \frac{1}{n^2}\Eb[\overset{n}{\underset{i=1}{\sum}}K_{ii}L_{ii}+\underset{(i,j) \in \mathbf{i}^n_2}{\sum} K_{ij}L_{ij}]
&=\frac{1}{n^2}\left(O(n)+(n)_2\Eb_{XYX'Y'}[k(X,X')l(Y,Y')]\right) \\ &= O(n^{-1})+\Eb_{XYX'Y'}[k(X,X')l(Y,Y')].
\end{split}
\end{eqnarray*}}}
As for $\widehat{S}^v_3$:
\begin{eqnarray*}
\begin{split}
\lefteqn{\Eb[\widehat{S}^v_3] = \frac{1}{n^3}\Eb[\overset{n}{\underset{i=1}{\sum}}K_{ii}L_{ii}+ \underset{(i,j) \in \mathbf{i}^n_2}{\sum}\Big(K_{ii}L_{ij}+K_{ij}L_{ii}\Big)+\underset{(i,j,m) \in \mathbf{i}^n_3}{\sum}K_{ij}L_{im}]} \\
&=  \frac{1}{n^3}\left(O(n)+O(n^2)+(n)_3\Eb_{XY}[\Eb_{X'}[k(X,X')]\Eb_{Y'}[l(Y,Y')]]\right) \\ 
&= O(n^{-1})+\Eb_{XY}[\Eb_{X'}[k(X,X')]\Eb_{Y'}[l(Y,Y')]].
\end{split}
\end{eqnarray*}
Finally, let us treat $\widehat{S}^v_2$ by the same reasoning:
\begin{eqnarray*}
\begin{split}
\Eb[\widehat{S}^v_2] &=\frac{1}{n^4}\left(O(n)+O(n^2)+O(n^3)+(n)_4\Eb_{XX'}[k(X,X')]\Eb_{YY'}[l(Y,Y')]\right) \\
&= O(n^{-1})+\Eb_{XX'}[k(X,X')]\Eb_{YY'}[l(Y,Y')]. 
\end{split}
\end{eqnarray*}
So $\text{HSIC}(X,Y) = \Eb[\widehat{\text{HSIC}}_v(\mathbf{X},\mathbf{Y})]+O(n^{-1})$, in line with Theorem 1 of \cite{ALT:Gretton+etal:2005}. Now, to derive the deviation inequality, we consider the difference $\widehat{\text{HSIC}}_v(\mathbf{U},\mathbf{V})-\widehat{\text{HSIC}}_u(\mathbf{U},\mathbf{V})$. Here, the U-statistic estimator of $\text{HSIC}$ can be written as a combination of three U-statistics defined in (\ref{hsic_u}) as follows: $\widehat{\text{HSIC}}_u(\mathbf{U},\mathbf{V}) =: \widehat{S}^u_1 + \widehat{S}^u_2 - 2 \widehat{S}^u_3$. Then: $$\widehat{\text{HSIC}}_v(\mathbf{U},\mathbf{V})-\widehat{\text{HSIC}}_u(\mathbf{U},\mathbf{V})=\widehat{S}^v_1-\widehat{S}^u_1+\widehat{S}^v_2-\widehat{S}^u_2-2(\widehat{S}^v_3-\widehat{S}^u_3).$$
We now treat each difference $\widehat{S}^v_1-\widehat{S}^u_1$, $\widehat{S}^v_2-\widehat{S}^u_2$ and $\widehat{S}^v_3-\widehat{S}^u_3$, again following Subsection A.1 of \cite{gretton2007kernel}.

\mds

\noindent First,
$\widehat{S}^v_1 = \frac{n-1}{n}\widehat{S}^u_1 + \frac{1}{n} \widehat{S}^u_0, \; \text{where} \; \widehat{S}^u_0 = \frac{1}{n}\overset{n}{\underset{i=1}{\sum}} K_{ii} L_{ii}$. As for $\widehat{S}^v_3$, we get
\begin{equation*}
\widehat{S}^v_3 = \widehat{S}^u_3 -\frac{3}{n}\widehat{S}^u_3 + \frac{1}{n^3}\underset{(i,j)\in \mathbf{i}^n_2}{\sum}\big( K_{ii}L_{ij}+K_{ij}L_{ii}+K_{ij}L_{ij}\big) + O(n^{-2}),
\end{equation*}
so we write $\widehat{S}^v_3 = \frac{n-3}{n}\widehat{S}^u_3 + \frac{n-1}{n^2}\widehat{T}^u + O(n^{-2})$, with $\widehat{T}^u := \frac{1}{n(n-1)}\underset{(i,j)\in \mathbf{i}^n_2}{\sum}\big( K_{ii}L_{ij}+K_{ij}L_{ii}+K_{ij}L_{ij}\big)$.
Finally, for $\widehat{S}^v_2$, we obtain:
\begin{equation*}
\widehat{S}^v_2 = \widehat{S}^u_2 -\frac{6}{n}\widehat{S}^u_2 + \frac{1}{n^4} \underset{(i,j,m)\in \mathbf{i}^n_3}{\sum}\big(K_{ii}L_{jm}+4K_{ij}L_{im}+K_{ij}L_{mm}\big)+O(n^{-2}),
\end{equation*}
so we write $\widehat{S}^v_2 = \frac{n-6}{n}\widehat{S}^u_2 + \frac{(n-1)(n-2)}{n^3} \widehat{P}^u + O(n^{-2})$, with $\widehat{P}^u := \frac{1}{n(n-1)(n-2)}\underset{(i,j,m)\in \mathbf{i}^n_3}{\sum}\big(K_{ii}L_{jm}+4K_{ij}L_{im}+K_{ij}L_{mm}\big)$.

\mds

\noindent First, we derive a deviation inequality for $|\widehat{S}^v_1-S_1|$, with $S_1 = \Eb_{XX'YY'}[k(X,X')l(Y,Y')]$. Under the assumption of non-negative and symmetric kernels $k(\cdot,\cdot), l(\cdot,\cdot)$ bounded by $1$, for any $0<\eps<1$, for $n$ such that $n \geq 1/\eps$, then $\widehat{S}^u_0/n\leq \eps$ and $S_1 /n \leq \eps$. We have
\begin{eqnarray*}
\Pb\Big(|\widehat{S}^v_1-S_1|\geq 3\eps\Big) \leq \Pb\Big(|\frac{n-1}{n}\widehat{S}^u_1 + \frac{1}{n} \widehat{S}^u_0 - S_1|\geq3\eps\Big)\leq \Pb\Big(|\frac{n-1}{n}\big(\widehat{S}^u_1-S_1\big)|\geq 2\eps-|\frac{1}{n}S_1|\Big) \\ \leq \Pb\Big(|\widehat{S}^u_1-S_1|\geq \eps\Big).
\end{eqnarray*}
By Theorem 5.6.1.A of \cite{serfling1980}, since the  bound holds for deviations in the opposite direction, we deduce
\begin{equation*}
\forall \eps >0, \; \Pb\Big(|\widehat{S}^u_1-S_1|\geq \eps\Big) \leq 2\exp\Big(-2\lfloor n/2\rfloor\eps^2\Big),
\end{equation*}
We now treat $\widehat{S}^v_3$. Take $0<\eps<1$, and $n$ such that $n \geq 3/\eps$, then $\frac{n-1}{n^2}\widehat{T}^u \leq \eps$ and $3S_3 / n\leq \eps$ with $S_3$ defined by $S_3 = \Eb_{XY}[\Eb_{X'}[\psi(X,X')]\Eb_{Y'}[\phi(Y,Y')]]$. We obtain for a bounded constant $K>0$:
\begin{eqnarray*}
\begin{split}
\lefteqn{\Pb\Big(|\widehat{S}^v_3-S_3|\geq 4\eps\Big)}\\
& \leq  \Pb\Big(|\frac{n-3}{n}\widehat{S}^u_3+\frac{n-1}{n^2}\widehat{T}^u+O(n^{-2})-S_3|\geq 4\eps\Big) \\ &\leq\Pb\Big(|\frac{n-3}{n}\big(\widehat{S}^u_3-S_3\big)| \geq 3\eps-\frac{3}{n}S_3-\frac{K}{n^2}\Big) \\ &\leq  \Pb\Big(|\widehat{S}^u_3-S_3|\geq \eps\Big).
\end{split}
\end{eqnarray*}
Using the same argument for bounding $\widehat{S}^v_1$, we deduce
\begin{equation*}
\Pb\Big(|\widehat{S}^v_3-S_3|\geq 4\eps\Big) \leq 2\exp\Big(-2\lfloor n/3 \rfloor \eps^2\Big).
\end{equation*}
Finally, we consider $\widehat{S}^v_2$. For $0 < \eps <1$, let $n$ such that $n \geq 6/\eps$, then $\frac{(n-1)(n-2)}{n^3}\widehat{P}^u \leq \eps$ and $6S_2 / n\leq \eps$ where $S_2 = \Eb_{XX'}[\psi(X,X')]\Eb_{YY'}[\phi(Y,Y')]$. We obtain for any $\eps>0$, for a bounded constant $K'>0$:
\begin{eqnarray*}
\lefteqn{\Pb\Big(|\widehat{S}^v_2-S_2|\geq 4\eps\Big)\leq \Pb\Big(|\frac{n-6}{n}\widehat{S}^u_2+\frac{(n-1)(n-2)}{n^3}\widehat{P}^u+O(n^{-2})-S_3|\geq 4\eps\Big)}\\
& \leq & \Pb\Big(|\frac{n-6}{n}\big(\widehat{S}^u_3-S_3\big)|\geq 3\eps-\frac{6}{n}S_3-\frac{K'}{n^2}\Big) \leq  \Pb\Big(|\widehat{S}^u_3-S_3|\geq \eps\Big) \leq 2\exp\Big(-2\lfloor n/4 \rfloor (\eta_1\eta_2)^2\Big).
\end{eqnarray*}
We then obtain, for any $\eps>0$:
\begin{eqnarray*}
\lefteqn{ \Pb\Big(|\widehat{\text{HSIC}}_v(\mathbf{X},\mathbf{Y}) - \text{HSIC}(X,Y)|\geq \eps\Big)\leq\Pb\Big(|\widehat{S}^v_1+\widehat{S}^v_2-2\widehat{S}^v_3 - (S_1+S_2-2S_3)|\geq \eps\Big)}\\
& \leq & \Pb\Big(|\widehat{S}^v_1-S_1|\geq \eps/3\Big)+ \Pb\Big(|\widehat{S}^v_2-S_2|\geq \eps/3\Big)+\Pb\Big(|\widehat{S}^v_3-S_3|\geq \eps/6\Big) \leq O\Big(\exp\left(-c_0 n \eps^2\right)\Big),
\end{eqnarray*}
with $c_0>0$ a finite constant. Finally, applying Lemma S.2 of \cite{liu2022} to the previous inequality (as $\text{HSIC}(X,X)>0, \text{HSIC}(X,Y)>0$ and $\widehat{S}^v_k,S_k$ with $k=1,2,3$, are bounded) and deduce the desired bound.
\end{proof}

\renewcommand{\thesection}{Appendix. \Alph{section}}
\renewcommand\theequation{B.\arabic{equation}} 
\section{Proofs}\label{proofs}

\subsection{Proof of Theorem \ref{bound_proba}}

Let $\nu_n =\sqrt{p_ns^2_n\log(p_n)}\big(n^{-1/2}+a_{n}\big)$, with $a_n = \underset{1\leq k \leq p_n}{\max}\big\{ \partial_2  \pp(\lambda_n,\theta_{n0,k}), \theta_{n0,k}\neq 0\big\}$ for the SCAD and MCP, and $a_n = \lambda_n$ for the LASSO. We aim to prove that, for any $\eps>0$, there exists $C_{\eps} > 0$ such that
\begin{equation}\label{objective_bound_proba}
\Pb\left(\|\widehat{\boldtheta}_n-\boldtheta_{n0}\|_2/\nu_n \geq C_{\eps}\right)  < \eps.
\end{equation}
Now, following \cite{fan2001variable}, Theorem 1, and denoting $n\,\Lb^{\text{pen}}_{v,n}(\boldtheta_n) = n\,\Lb_{v,n}(\boldtheta_n) + n\overset{p_n}{\underset{k=1}{\sum}}\pp(\lambda_n,|\theta_{nk}|)$, we have
\begin{equation}
\Pb\left(\|\widehat{\boldtheta}_n-\boldtheta_{n0}\|_2/\nu_n \geq C_{\eps}\right) \leq \Pb\left(\exists \uu \in \Rb^{p_n}_+,\|\uu\|_2=C'_{\eps}\geq  C_\eps: n\,\Lb^{\text{pen}}_{v,n}(\boldtheta_{n0}+\nu_n \uu) \leq n\,\Lb^{\text{pen}}_{v,n}(\boldtheta_{n0})\right),
\label{ineg1_FanLi}
\end{equation}
and we can set $C'_\eps=C_\eps$, our choice hereafter.
If the r.h.s. of~(\ref{ineg1_FanLi}) is less than $\eps$, there is a local minimum in the ball $\big\{\boldtheta_{n0}+\nu_n \uu, \|\uu\|_2 \leq C_{\eps}\big\}$ with a probability larger than $1-\eps$. So, (\ref{objective_bound_proba}) is satisfied, $\|\widehat{\boldtheta}_n-\boldtheta_{n0}\|_2 = O_p(\nu_n)$. We have
\begin{eqnarray*}
\begin{split}
n\,\Lb^{\text{pen}}_{v,n}(\boldtheta_{n0}+\nu_n\uu)-n\,\Lb^{\text{pen}}_{v,n}(\boldtheta_{n0})
& = n\,\Lb_{v,n}(\boldtheta_{n0}+\nu_n\uu)-n\,\Lb_{v,n}(\boldtheta_{n0})  \\ &+ n\overset{p_n}{\underset{k=1}{\sum}}\Big\{\pp(\lambda_n,\theta_{n0,k}+\nu_nu_k)-\pp(\lambda_n,\theta_{n0,k})\Big\} \\
& \geq n\,\Lb_{v,n}(\boldtheta_{n0}+\nu_n\uu)-n\,\Lb_{v,n}(\boldtheta_{n0}) \\ &+ n\underset{k\in \Sc_n}{\sum}\Big\{\pp(\lambda_n,\theta_{n0,k}+\nu_nu_k)-\pp(\lambda_n,\theta_{n0,k})\Big\}. 
\end{split}
\end{eqnarray*}
We have used $\pp(\lambda_n,0)=0$. Then, we have the decomposition
\begin{eqnarray*}
\lefteqn{\Lb_{v,n}(\boldtheta_{n0}+\nu_n\uu)-\Lb_{v,n}(\boldtheta_{n0}) = -  \overset{p_n}{\underset{k=1}{\sum}} (\theta_{n0,k}+\nu_n\uu_k) \widehat{\text{D}}_v(\mathbf{X}_k,\mathbf{Y}) }\\
& & + \frac{1}{2} \overset{p_n}{\underset{k,l=1}{\sum}} (\theta_{n0,k}+\nu_n\uu_k)(\theta_{n0,l}+\nu_n\uu_l) \widehat{\text{D}}_v(\mathbf{X}_k,\mathbf{X}_l) + \overset{p_n}{\underset{k=1}{\sum}} \theta_{n0,k} \widehat{\text{D}}_v(\mathbf{X}_k,\mathbf{Y}) \\ & & - \frac{1}{2} \overset{p_n}{\underset{k,l=1}{\sum}} \theta_{n0,k}\theta_{n0,l} \widehat{\text{D}}_v(\mathbf{X}_k,\mathbf{X}_l)\\
& = & -\nu_n\uu^\top\Jb_{v,n}+\nu_n\uu^\top\KK_{v,n}\boldtheta_{n0}+ \frac{\nu^2_n}{2}\uu^\top\KK_{v,n}\uu,
\end{eqnarray*}
where $\Jb_{v,n} = (\widehat{\text{D}}_{v}(\mathbf{X}_1,\mathbf{Y}),\ldots,\widehat{\text{D}}_{v}(\mathbf{X}_p,\mathbf{Y}))^\top \in \Rb^{p_n}, \; \KK_{v,n} = (\widehat{\text{D}}_{v}(\mathbf{X}_k,\mathbf{X}_l))_{1\leq k\leq l \leq p_n} \in \Rb^{p_n \times p_n}$. Therefore, it is sufficient to prove that there exists $C_{\eps}$ such that
\begin{equation}
\Pb\left(\exists \uu \in \Rb^{p_n}_+, \|\uu\|_2=C_{\eps}:
nT_1+nT_2+nT_3 \leq 0\right) < \eps, \label{bound_obj_double}
\end{equation}
with $T_1 :=  -\nu_n\uu^\top\Jb_{v,n}+\nu_n\uu^\top\KK_{v,n}\boldtheta_{n0}$, $T_2:= \frac{\nu^2_n}{2}\uu^\top\KK_{v,n}\uu$, $T_3:=\sum_{k\in \Sc_n}\Big\{\pp(\lambda_n,\theta_{n0,k}+\nu_nu_k)-\pp(\lambda_n,\theta_{n0,k})\Big\}$. Hereafter, we divide by $n$ in (\ref{bound_obj_double}). First, note that $|T_1| \leq \nu_n \|\uu\|_2\|\KK_{v,n}\boldtheta_{n0}-\Jb_{v,n}\|_2$. By Assumption \ref{regularity_condition}, $\partial_{\boldtheta_{n,k}}\Lb(\boldtheta_{n0})=0$ for any $1 \leq k \leq p_n$, so for any $a>0$:
\begin{eqnarray*}
\lefteqn{\Pb\Big(\underset{\uu:\|\uu\|_2=C_{\eps}}{\sup}|T_1|>a\Big) \leq \Pb\Big(\underset{\uu:\|\uu\|_2=C_{\eps}}{\sup}\nu_n \|\uu\|_2 \sqrt{p_n}\|\KK_{v,n}\boldtheta_{n0}-\Jb_{v,n} \|_{\infty}>a \Big) }\\
& = & \Pb\Big(\underset{\uu:\|\uu\|_2=C_{\eps}}{\sup} \nu_n \|\uu\|_2 \sqrt{p_n}\underset{1 \leq j \leq p_n}{\max}|-\widehat{\text{D}}_{v}(\mathbf{X}_j,\mathbf{Y})+\text{D}(X_j,Y) + \overset{p_n}{\underset{l=1}{\sum}}\big(\widehat{\text{D}}_{v}(\mathbf{X}_j,\mathbf{X}_l)-\text{D}(X_j,X_l)\big)\theta_{n0,l}|>a\Big) \\
& \leq & \Pb\Big(\underset{\uu:\|\uu\|_2=C_{\eps}}{\sup} \nu_n \|\uu\|_2 \sqrt{p_n} \underset{1 \leq j \leq p_n}{\max}|\widehat{\text{D}}_{v}(\mathbf{X}_j,\mathbf{Y})-\text{D}(X_j,Y) |>a/2\Big)\\
& & + \Pb\Big(\underset{\uu:\|\uu\|_2=C_{\eps}}{\sup} \nu_n \|\uu\|_2 \sqrt{p_n}\underset{1 \leq j \leq p_n}{\max}|\overset{p_n}{\underset{l=1}{\sum}}\big(\widehat{\text{D}}_{v}(\mathbf{X}_j,\mathbf{X}_l)-\text{D}(X_j,X_l)\big)\theta_{n0,l}|>a/2\Big).
\end{eqnarray*}
Using the deviation inequalities $|\widehat{\text{D}}_{v}(\mathbf{X}_j,\mathbf{Y})-\text{D}(X_j,Y) |>\eps$, for any $\eps>0$, provided in Section \ref{deviation}, we deduce by Bonferroni's inequality:
\begin{eqnarray*}
\lefteqn{\Pb\Big(\underset{\uu:\|\uu\|_2=C_{\eps}}{\sup} \nu_n \|\uu\|_2 \sqrt{p_n} \underset{1 \leq j \leq p_n}{\max}|\widehat{\text{D}}_{v}(\mathbf{X}_j,\mathbf{Y})-\text{D}(X_j,Y) |>a/2\Big)}\\
& \leq & \Pb\Big(\underset{1 \leq j \leq p_n}{\max}|\widehat{\text{D}}_{v}(\mathbf{X}_j,\mathbf{Y})-\text{D}(X_j,Y) |>a/(2C_{\eps}\nu_n\sqrt{p_n})\Big)
\leq p_n\, O\Big(\exp(-c_1n\frac{a^2}{4C^2_{\eps}\nu^2_np_n})\Big).
\end{eqnarray*}
Second, we have $\sum^{p_n}_{l=1}\big(\widehat{\text{D}}_{v}(\mathbf{X}_j,\mathbf{X}_l)-\text{D}(X_j,X_l)\big)\boldtheta_{n0,l} =\sum_{l \in \Sc_n}\big(\widehat{\text{D}}_{v}(\mathbf{X}_j,\mathbf{X}_l)-\text{D}(X_j,X_l)\big)\theta_{n0,l}$ for any $j = 1,\ldots,p_n$, due to the sparsity assumption, that is, without loss of generality, $\boldtheta_{n0}$ can be re-written as $ \boldtheta_{n0}= (\theta_{n0,1},\ldots,\theta_{n0,s_n},\theta_{n0,(s_n+1)},\ldots,\theta_{n0,p_n})^\top$, where $\theta_{n0,k} \neq 0$ for $k \leq s_n$ and $\theta_{n0,k}=0$ for all $k=s_n+1,\ldots,p_n$. Let $L_{nj} = \big(\widehat{\text{D}}_{v}(\mathbf{X}_j,\mathbf{X}_l)-\text{D}(X_j,X_l)\big)_{l \in \Sc_n}\in \Rb^{s_n}$. Then, by Cauchy-Schwarz inequality and by Bonferroni's inequality:
\begin{eqnarray*}
\lefteqn{\Pb\Big(\underset{\uu:\|\uu\|_2=C_{\eps}}{\sup} \nu_n \|\uu\|_2 \sqrt{p_n}\underset{1 \leq j \leq p_n}{\max}|\overset{p_n}{\underset{l=1}{\sum}}\big(\widehat{\text{D}}_{v}(\mathbf{X}_j,\mathbf{X}_l)-\text{D}(X_j,X_l)\big)\theta_{n0,l}|>a/2\Big)} \\
& \leq & \Pb\Big(\underset{1 \leq j \leq p_n}{\max}|{\sum}_{l\in \Sc_n}\big(\widehat{\text{D}}_{v}(\mathbf{X}_j,\mathbf{X}_l)-\text{D}(X_j,X_l)\big)\theta_{n0,l}|>a/(2C_{\eps}\nu_n\sqrt{p_n})\Big)\\
& \leq & \Pb\Big(\underset{1 \leq j \leq p_n}{\max}\big(\|L_{nj}\|_2\|\boldtheta_{n0}\|_2\big)>a/(2C_{\eps}\nu_n\sqrt{p_n})\Big)\\
& \leq & \Pb\Big(s_n\underset{1 \leq j \leq p_n}{\max}\,\underset{l \in \Sc_n}{\max}\,|\widehat{\text{D}}_{v}(\mathbf{X}_j,\mathbf{X}_l)-\text{D}(X_j,X_l)|\|\boldtheta_{n0}\|_{\infty}>a/(2C_{\eps}\nu_n\sqrt{p_n})\Big) \\
& \leq & s_n\,p_n\, O\Big(\exp(-c_1n\frac{a^2}{4C^2_{\eps}\nu^2_np_ns^2_n\|\boldtheta_{n0}\|^2_{\infty}})\Big).
\end{eqnarray*}
Putting the pieces together, we obtain
\begin{eqnarray*}
\Pb\Big(\underset{\uu:\|\uu\|_2=C_{\eps}}{\sup}|T_1|>a\Big)& \leq & p_n\, O\Big(\exp(-c_1n\frac{a^2}{4C^2_{\eps}\nu^2_np_n})\Big) + s_n\,p_n \,O\Big(\exp(-c_1n\frac{a^2}{4C^2_{\eps}\nu^2_np_ns^2_n\|\boldtheta_{n0}\|^2_{\infty}})\Big).
\end{eqnarray*}
Furthermore, we have $T_2 = \frac{\nu^2_n}{2}\uu^\top \text{D}_{\mathbf{X}\mathbf{X}}\uu + \frac{\nu^2_n}{2}\uu^\top \mathcal{R}_n$, $\mathcal{R}_n := \sum^{p_n}_{k,l=1}\uu_k\uu_l\big(\widehat{\text{D}}_{v}(\mathbf{X}_k,\mathbf{X}_l)-\text{D}(X_k,X_l)\big)$. For any $b>0$, using $\|A\|_s \leq p_n \|A\|_{\max}$ for $A \in \Rb^{p_n \times p_n}$, by Bonferroni's inequality, we have
\begin{eqnarray*}
\begin{split}
\Pb\Big(\underset{\uu:\|\uu\|_2=C_{\eps}}{\sup}|\mathcal{R}_n|>b\Big) & \leq \Pb\Big(\underset{\uu:\|\uu\|_2=C_{\eps}}{\sup}\|\uu\|^2_2\,p_n\,\underset{1\leq k,l \leq p_n}{\max}|\widehat{\text{D}}_{v}(\mathbf{X}_k,\mathbf{X}_l)-\text{D}(X_k,X_l)|>b\Big) \\
&\leq p^2_n O\Big(\exp(-c_1 n \frac{b^2}{C^4_{\eps}p^2_n})\Big).
\end{split}
\end{eqnarray*}
Let us now consider the penalization part. By Assumption \ref{assumption_penalty}, we obtain the bound
\begin{equation*}
|\underset{k \in \Sc_n}{\sum} \pp(\lambda_n,\theta_{n0,k}+ \nu_n \uu_{k})-\pp(\lambda_n,\theta_{n0,k}) | \leq \nu_n \|\uu\|_1 a_n + \frac{\nu^2_n}{2} \|\uu\|^2_2 b_n \big(1+o(1)\big).
\end{equation*}
Using $\|\uu\|_1 \leq \sqrt{s_n} \|\uu\|_2$, we get $|\underset{k \in \Sc_n}{\sum} \pp(\lambda_n,\theta_{n0,k}+ \nu_n  \uu_{k})-\pp(\lambda_n,\theta_{n0,k}) | \leq \nu_n \sqrt{s_n} \|\uu\|_2 a_n + \nu^2_n \|\uu\|^2_2 b_n$.\\
In the case of the LASSO penalty, we have
\begin{eqnarray*}
|\lambda_n \underset{k \in \Sc_n}{\sum}|\big\{|\theta_{n0,k}+\nu_n\uu_k|-|\theta_{0n,k}|\big\}|\leq\lambda_n \underset{k \in \Sc_n}{\sum}\nu_n|\uu_k|O_p(1) \leq \lambda_n  \nu_n\sqrt{s_n}\|\uu\|_2.
\end{eqnarray*}
Now, let $\delta_n = \lambda_{\min}(\text{D}_{\mathbf{X}\mathbf{X}}) C^2_{\eps} \nu^2_n/2$, where $\lambda_{\min}(\text{D}_{\mathbf{X}\mathbf{X}})>0$ by Assumption \ref{eigenvalue_assumption}. For $n$, $C_{\eps}$ sufficiently large, for $a_1>0$ finite:
\begin{eqnarray*}
\lefteqn{\Pb\Big(\exists \uu \in \Rb^{p_n}, \|\uu\|_2=C_{\eps}:| \nu_n\uu^\top\big(\KK_{v,n}\boldtheta_{n0}-\Jb_{v,n}\big)|>a_1\delta_n\Big)}\\
& \leq & O\Big(\exp(-c_1n\frac{a^2_1\delta^2_n}{4C^2_{\eps}\nu^2_np_n}+\log(p_n))\Big) + O\Big(\exp(-c_1 n\frac{a^2_1\delta^2_n}{4C^2_{\eps}\nu^2_np_ns^2_n\|\boldtheta_{n0}\|^2_{\infty}}+\log(p_n)+\log(s_n))\Big)\\
& \leq & O\Big(\exp(-c'_1n\frac{C^4_{\eps}\nu^4_n}{C^2_{\eps}\nu^2_np_n}+\log(p_n))\Big) + O\Big(\exp(-c''_1 n\frac{C^4_{\eps}\nu^4_n}{C^2_{\eps}\nu^2_np_ns^2_n\|\boldtheta_{n0}\|^2_{\infty}}+\log(p_n)+\log(s_n))\Big)\\
& \leq & O\Big(\exp(-c'_1n\frac{C^2_{\eps}\nu^2_n}{p_n}+\log(p_n))\Big) + O\Big(\exp(-c''_1 n\frac{C^2_{\eps}\nu^2_n}{p_ns^2_n\|\boldtheta_{n0}\|^2_{\infty}}+2\log(p_n))\Big)\\
& \leq & O\Big(\exp(-c'_1n\frac{C^2_{\eps}p_ns^2_n\log(p_n)n^{-1}}{p_n}+\log(p_n))\Big)+O\Big(\exp(-c''_1n\frac{C^2_{\eps}p_ns^2_n\log(p_n)n^{-1}}{p_ns^2_n\|\boldtheta_{n0}\|^2_{\infty}}+2\log(p_n))\Big),
\end{eqnarray*}
which will tend to zero, where $c'_1, c''_1>0$ are finite constant. So we deduce $|T_1| \leq \nu_n \|\uu\|_2\|\KK_{v,n}\boldtheta_{n0}-\Jb_{v,n}\|_2 = O_p(\delta_n)$.
As for $T_2$, for $C_{\eps}>0$ large enough, we have for $a_2>0$ a finite constant:
\begin{eqnarray*}
\begin{split}
\Pb\Big(\exists \uu \in \Rb^{p_n}, \|\uu\|_2=C_{\eps}:|\frac{\nu^2_n}{2}\mathcal{R}_n|>a_2\delta_n\Big)
&\leq O\Big(\exp(-c_1n\frac{a^2_2\delta^2_n}{4\nu^4_nC^4_{\eps}p^2_n}+2\log(p_n))\Big) \\ &=O\Big(\exp(-c'''_1\frac{n}{p^2_n}+2\log(p_n))\Big),
\end{split}
\end{eqnarray*}
with $c'''_1>0$, so that under $p^2_ns_n\log(p_n) = o(n)$, the right-hand side tends to zero. Hence, $T_2 = \frac{\nu^2_n}{2}\uu^\top \text{D}_{\mathbf{X}\mathbf{X}}\uu + O_p(\delta_n)$.
Finally, we have $|T_3| \leq \sqrt{s_n} \nu_n a_n \|\uu\|_2 + \frac{\nu^2}{2}b_n \|\uu\|^2_2(1+o(1))$.
Therefore, we get $nT_1+nT_2+nT_3 = n \frac{\nu^2_n}{2} \uu^\top \text{D}_{\mathbf{X}\mathbf{X}}\uu (1+o_p(1))$.
As a consequence, since $n\frac{\nu^2_n}{2} \uu^\top \text{D}_{\mathbf{X}\mathbf{X}}\uu \geq n\delta_n$, with $C^2_{\eps}$ large enough, (\ref{bound_obj_double}) is satisfied, which implies the probability bound $\|\widehat{\boldtheta}_n-\boldtheta_{n0}\|_2=O_p(\nu_n)$.

\subsection{Proof of Theorem \ref{sparsistency}}

To prove the recovery of the zero entries, we show that with probability tending to one, for any $\boldtheta_{n1}$ such that $\|\boldtheta_{n1}-\boldtheta_{n01}\|_2=O_p(\sqrt{p_ns^2_n\log(p_n)/n})$, and any constant $C>0$:
\begin{equation*}
n\,\Lb^{\text{pen}}_{v,n}((\boldtheta^\top_{n1},0^\top)^\top) = \underset{\|\boldtheta_{n2}\|_2 \leq C \sqrt{p_ns^2_n\log(p_n)/n}}{\min} n\,\Lb^{\text{pen}}_{v,n}((\boldtheta^\top_{n1},\boldtheta^\top_{n2})^\top).
\end{equation*}
Let $u_n = C \sqrt{p_ns^2_n\log(p_n)/n}$. To prove this statement, we show that for any $\boldtheta_{n}\in [0,+\infty[^n$ such that $\|\boldtheta_{n}-\boldtheta_{n0}\|_2 \leq u_n$, we have with probability one that $\partial_{\theta_{n,j}}n\,\Lb^{\text{pen}}_{v,n}(\boldtheta_n)>0, \;\; \text{for all} \;\; j \in \Sc^c_n, 0<\theta_{n,j}<u_n$.
By an expansion of this derivative:
\begin{eqnarray*}
\begin{split}    
n\,\partial_{\theta_{n,j}}\Lb^{\text{pen}}_{v,n}(\boldtheta_n) &=  n\,\partial_{\theta_{n,j}}\Lb_{v,n}(\boldtheta_{n0}) + n\overset{p_n}{\underset{k=1}{\sum}}\partial^2_{\theta_{n,j}\theta_{n,k}}\Lb_{v,n}(\boldtheta_{n0})(\theta_{n,k}-\theta_{n0,k}) + n\,\partial_2\pp(\lambda_n,\theta_{n,j}) \\ &=:nT_1+nT_2+nT_3.
\end{split}
\end{eqnarray*}
Under Assumption \ref{regularity_condition}, $-\text{D}(X_j,Y)+\sum^{p_n}_{l=1}\text{D}(X_j,X_l)\theta_{n0,l}=0$ for any $1 \leq j \leq p_n$. So, for any $a>0$:
\begin{eqnarray*}
\Pb\big(|T_1|>a\big)
\leq \Pb\Big(|\widehat{\text{D}}_v(\mathbf{X}_j,\mathbf{Y}))-\text{D}(X_j,Y)|>a/2\Big) + \Pb\Big(|\overset{p_n}{\underset{l=1}{\sum}}\big(\widehat{\text{D}}_{v}(\mathbf{X}_j,\mathbf{X}_l)-\text{D}(X_j,X_l)\big)\theta_{n0,l}|>a/2\Big).
\end{eqnarray*}
Here, for any $j \in \Sc^c_n$, under the sparsity assumption:
{\small{\begin{eqnarray*}\begin{split}
\Pb\Big(|\overset{p_n}{\underset{l=1}{\sum}}\Big(\widehat{\text{D}}_{v}(\mathbf{X}_j,\mathbf{X}_l)-\text{D}(X_j,X_l)\Big)\theta_{n0,l}|>a/2\Big)
& \leq \Pb\Big(\|L_{nj}\|_2\|\boldtheta_{n0}\|_2>a/2\Big) \\ & \leq \Pb\Big(\sqrt{s_n}\underset{l \in \Sc_n}{\max}|\widehat{\text{D}}_{v}(\mathbf{X}_j,\mathbf{X}_l)-\text{D}(X_j,X_l)|\sqrt{s_n}\|\boldtheta_{n0}\|_{\infty}>a/2\Big).
\end{split}
\end{eqnarray*}}}
We deduce
\begin{equation*}
\Pb\big(|T_1|>a\big) \leq O\Big(\exp\big(-c_1na^2\big)\Big) + s_n\,O\Big(\exp\Big(-c_1n\frac{a^2}{4s^2_n\|\boldtheta_{n0}\|^2_{\infty}}\Big)\Big).
\end{equation*}
As for $T_2$, we have $T_2 = \overset{p_n}{\underset{k=1}{\sum}}\big(\partial^2_{\theta_{n,j}\theta_{n,k}}\Lb_{v,n}(\boldtheta_n)-\text{D}_{n,\mathbf{X}\mathbf{X},jk}\big)(\theta_{n,k}-\theta_{n0,k})+\overset{p_n}{\underset{k=1}{\sum}}\text{D}_{n,\mathbf{X}\mathbf{X},jk}(\theta_{n,k}-\theta_{n0,k})=:K_1+K_2$.
By the Cauchy–Schwarz inequality and $\|\boldtheta_{n}-\boldtheta_{n0}\|_2=O_p(\sqrt{p_ns^2_n\log(p_n)/n})$:
\begin{equation*}
|K_2|=|\overset{p_n}{\underset{k=1}{\sum}}\text{D}_{n,\mathbf{X}\mathbf{X},jk}(\theta_{n,k}-\theta_{n0,k})|\leq O_p(\sqrt{p_ns^2_n\log(p_n)/n})\Big[\overset{p_n}{\underset{k=1}{\sum}}\text{D}^2_{n,\mathbf{X}\mathbf{X},jk}\Big]^{1/2}.
\end{equation*}
Under the bounded eigenvalue of Assumption \ref{eigenvalue_assumption}, $\sum^{p_n}_{k=1}\text{D}^2_{n,\mathbf{X}\mathbf{X},jk}=O(1)$, so that $K_2 = O_p\big(\sqrt{p_ns^2_n\log(p_n)/n}\big)$. Moreover, let $U_{nj}$ be the vector
$U_{nj}=\big(\widehat{\text{D}}_v(\mathbf{X}_j,\mathbf{X}_k)-\text{D}(X_j,X_k)\big)_{1\leq k \leq p_n}\in \Rb^{p_n}$. B the Cauchy–Schwarz inequality, for any $b>0$:
\begin{eqnarray*}
\begin{split}
\lefteqn{\Pb\big(|K_1|>b\big)\leq  \Pb\Big(\|\boldtheta_{n}-\boldtheta_{n0}\|_2\|U_{nj}\|_2>b\Big)}\\
& \leq \Pb\Big(C_{\eps}\sqrt{\frac{p_ns^2_n\log(p_n)}{n}}\sqrt{p_n}\underset{1 \leq k \leq p_n}{\max}|\widehat{\text{D}}_v(\mathbf{X}_j,\mathbf{X}_k)-\text{D}(X_j,X_k)|>b\Big) \\ &+ \Pb\Big(\|\boldtheta_{n}-\boldtheta_{n0}\|_2> C_{\eps}\sqrt{\frac{p_ns^2_n\log(p_n)}{n}}\Big)\\
& \leq  p_n \, O\Big(\exp\Big(-c_1n\frac{nb^2}{p^2_ns^2_n\log(p_n)C^2_{\eps}}\Big)\Big)+\Pb\Big(\|\boldtheta_{n}-\boldtheta_{n0}\|_2> C_{\eps}\sqrt{\frac{p_ns^2_n\log(p_n)}{n}}\Big).
\end{split}
\end{eqnarray*}
As a consequence, we deduce $K_1=O_p\big(\sqrt{p_ns^2_n\log(p_n)n}\big)$, and $T_2=O_p\big(\sqrt{p_ns^2_n\log(p_n)n}\big)$. Therefore, we get
\begin{equation*}
\forall j \in \Sc^c_n, \; n\,\partial_{\theta_{n,j}}\Lb^{\text{pen}}_{v,n}(\boldtheta_n) = n\lambda_n\left[O_p\left(\sqrt{\frac{p_ns^2_n\log(p_n)}{n}}\lambda^{-1}_n\right)+\lambda^{-1}_n\partial_2\pp(\lambda_n,\theta_{n,j})\right].
\end{equation*}
Under $\underset{n \rightarrow \infty}{\lim \, \inf} \; \underset{x \rightarrow 0^+}{\lim \, \inf} \; \lambda^{-1}_n \partial_2\pp(\lambda_n,x) > 0$ and $\sqrt{n/(p_ns^2_n\log(p_n))}\lambda_n\rightarrow \infty$, we get $\partial_{\theta_{n,j}}n\,\Lb^{\text{pen}}_{v,n}(\boldtheta_n)>0$, $\forall j \in \Sc^c_n$.

\subsection{Proof of Theorem \ref{fdr_control}}

The way $\widehat{W}_k$ is built determines the success of the FDR control. More precisely, it is employed as a signal for active/inactive feature: if $\widehat{W}_k$ takes a sufficiently large positive value, then $X_k$ is active; $\widehat{W}_k$ is expected to take a small positive or negative value with the same probability when $X_k$ is inactive. Since $\widehat{W_k} = \widehat{\theta}_{n,k}-\widetilde{\theta}_{n,k}$, with $\widehat{\theta}_{n,k}, \widetilde{\theta}_{n,k}$ the solutions of penalized OLS problems, we may apply Lemma 1 of \cite{liu2022}. \\
The proof of the FDR control follows the same steps as in the proof of Theorem 4 of \cite{liu2022}.

\clearpage
\renewcommand{\thesection}{S\arabic{section}}
\renewcommand\theequation{S\arabic{equation}} 
\renewcommand{\theremark}{S\arabic{remark}}
\renewcommand{\theassumption}{S\arabic{assumption}}
\renewcommand{\thefigure}{S.\arabic{figure}}

\begin{center}
    {\huge  Supplementary Material to ``Sparse minimum Redundancy Maximum Relevance for feature selection''}
\\ \vspace{1cm}

  {\Large Peter~Naylor, Benjamin~Poignard, Héctor~Climente-González, Makoto~Yamada}
       
\end{center}

\setcounter{section}{0}
\setcounter{figure}{0}

\section{Experimental results for the simulated data} \label{app:exp}

The results of this section relate to the experimental setup for the simulated data described in Subsection \ref*{s:data} of the main text. 
In particular, we show the Accuracy (Acc), when categorical data, the Mean Squared Error, when continuous data, the true positive rate (TPR), the false positive rate (FPR), the false detection rate (FDR) and the number of selected features ($N$). For Acc and TPR, it is the higher the better. For the other, it is the lower the better.
PC denotes the projection correlation and our method is SmRMR for Sparse mRMR. 
% For all the figures of this supplement, we use the following nomenclature:
% \begin{itemize}
%     % \item $n$ is number of samples and $p$ is the number of features.
%     % \item AM: Association Measure.
%     \item Acc: Accuracy. The higher the better.
%     \item MSE: mean square error. The lower the better.
%     \item TPR: True Positive Rate. The higher the better.
%     \item FPR: False Positive Rate. The lower the better.
%     \item FDR: False Detection Rate. The lower the better.
%     \item $N$ is the number of selected features.
%     \item PC: projection correlation and our method is SmRMR: Sparse mRMR.
% \end{itemize}
We use the Gaussian kernel with the HSIC association measure.
% AM: Association Measure, Acc: Accuracy, MSE: mean square error, TPR: True Positive Rate, FPR: False Positive Rate, FDR: False Detection Rate, N is the number of selected features, PC: projection correlation and our method is SmRMR: Sparse mRMR.
% For Acc and TPR: the higher the better, for MSE, FPR and FDR: the lower the better. In the figures, we use the Gaussian kernel with the HSIC association measure.

\newpage

\begin{figure}[H]
\centering
\includegraphics[width=0.8\textwidth]{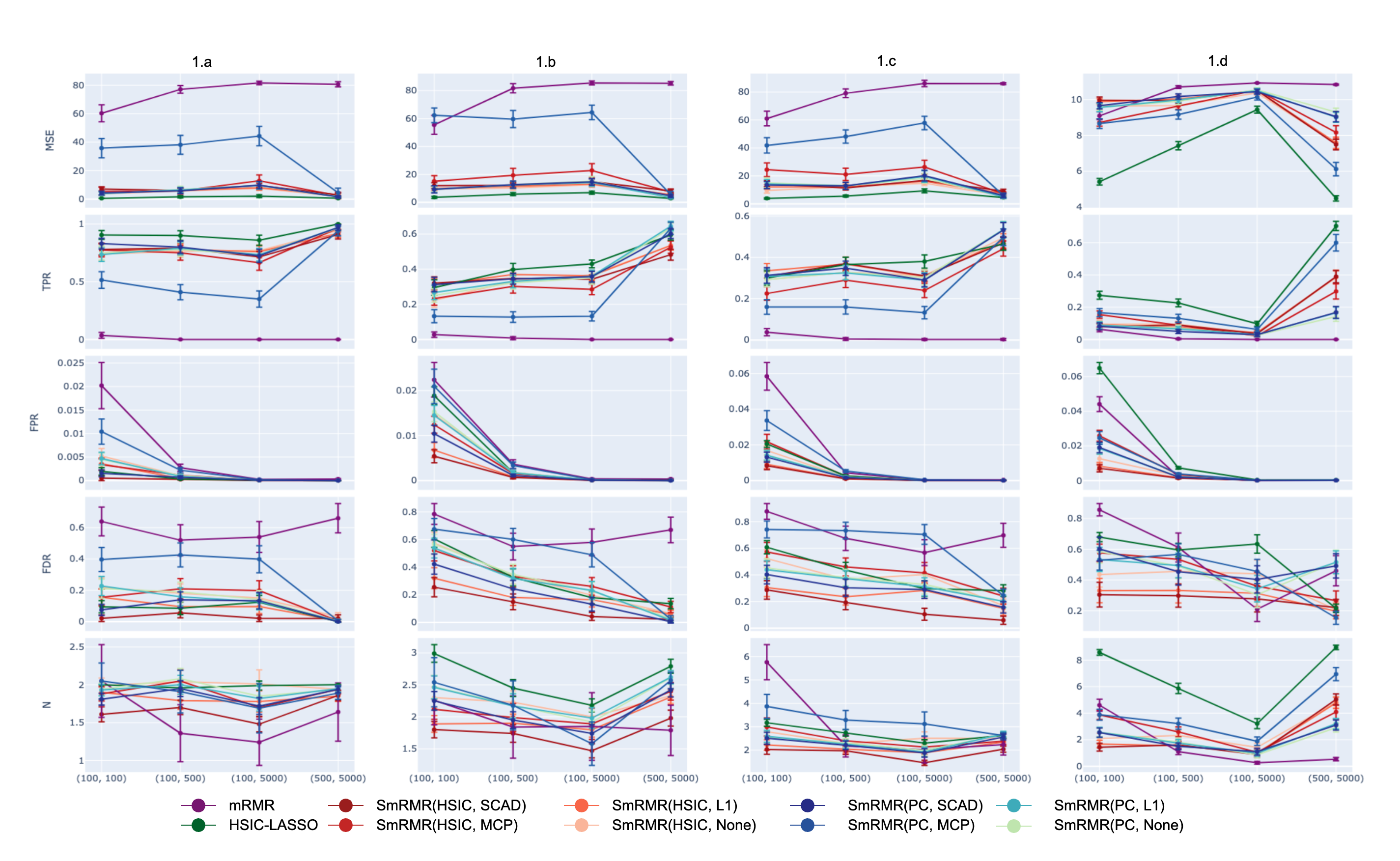}
\caption{Full results for the linear DGP with SmRMR\textsubscript{2}.}
\label{fig:linearDGP}
\end{figure}

\begin{figure}[H]
\centering
\includegraphics[width=0.8\textwidth]{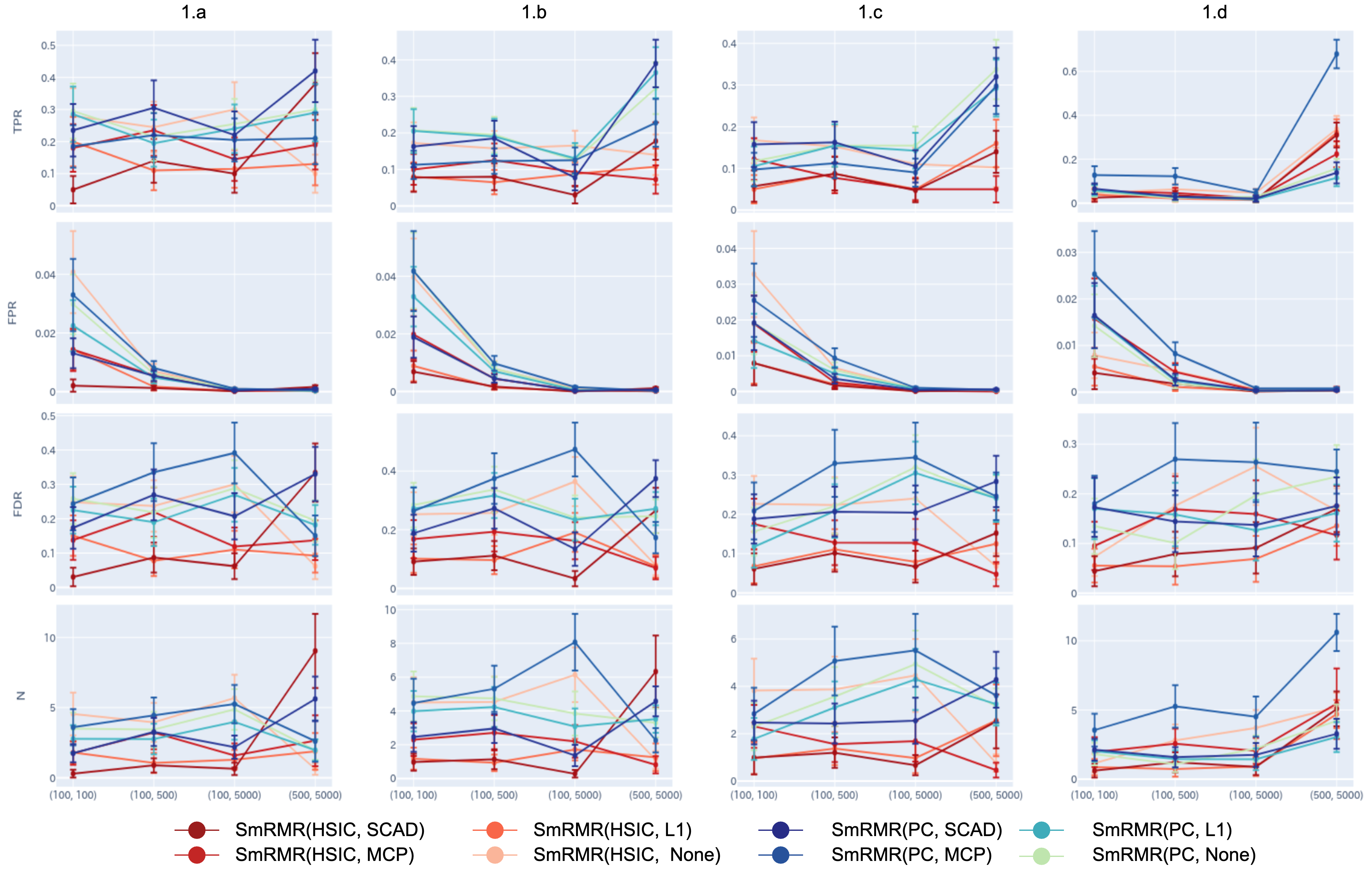}
\caption{Full results for the linear DGP with the unmodified SmRMR.}
\label{fig:linearDGP_smrmr}
\end{figure}

\begin{figure}[H]
\centering
\includegraphics[width=0.8\textwidth]{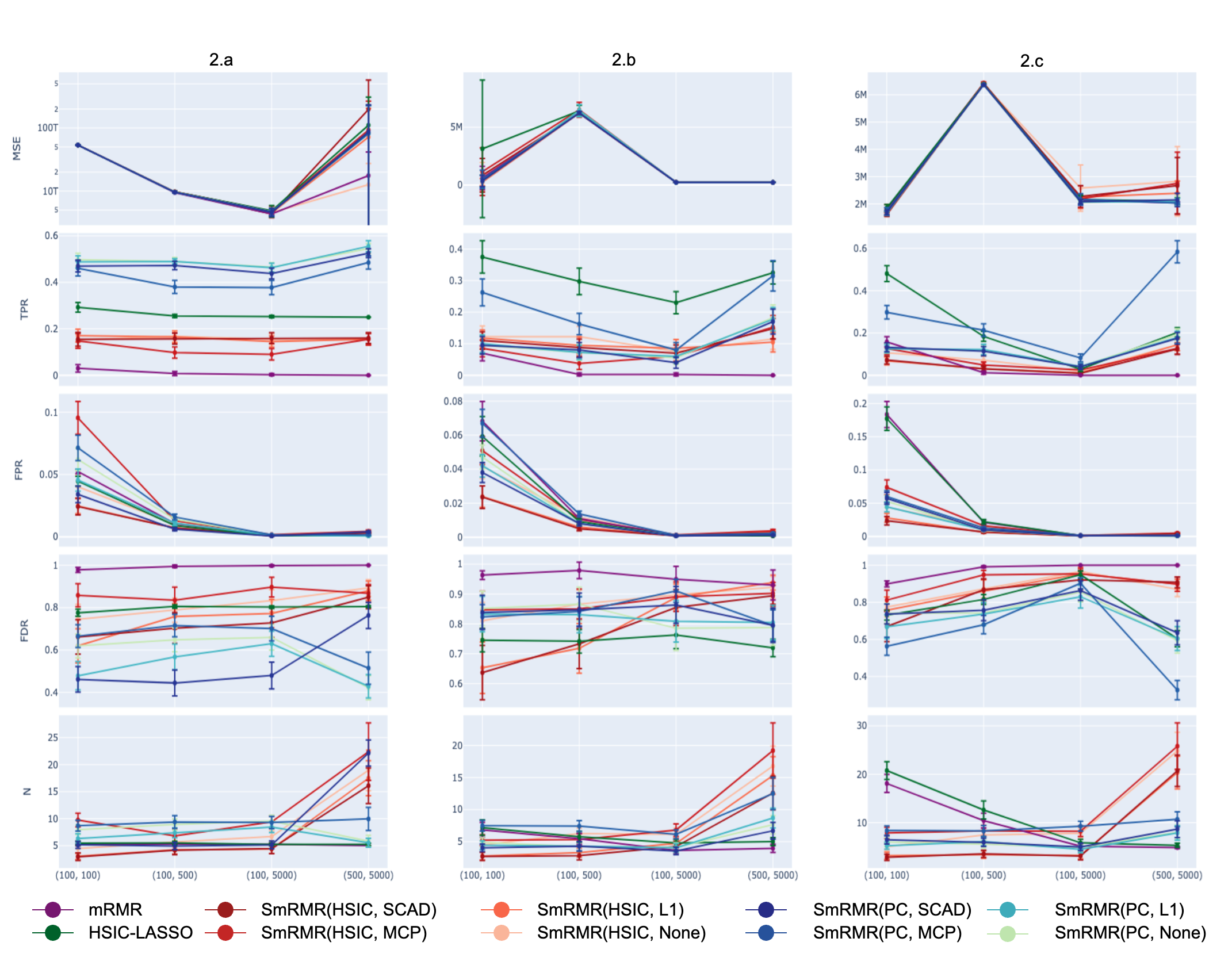}
\caption{Full results for the non-linear DGP with SmRMR\textsubscript{2}.}
\label{fig:FnonlinearDGP}
\end{figure}

\begin{figure}[H]
\centering
\includegraphics[width=0.8\textwidth]{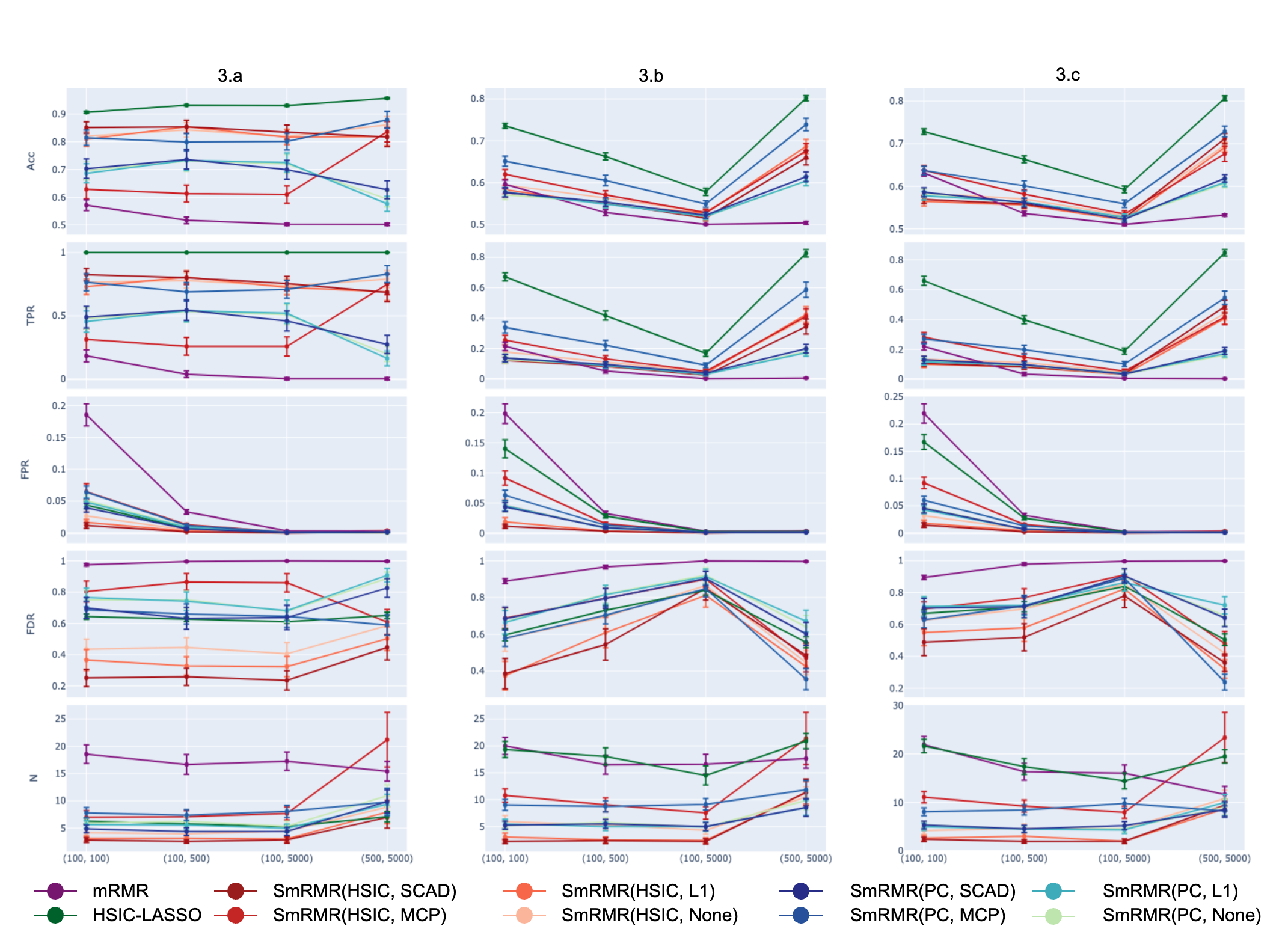}
\caption{Full results for the categorical DGP with the unmodified SmRMR\textsubscript{2}.}
\label{fig:catDGP}
\end{figure}

\begin{figure}[H]
\centering
\includegraphics[width=0.8\textwidth]{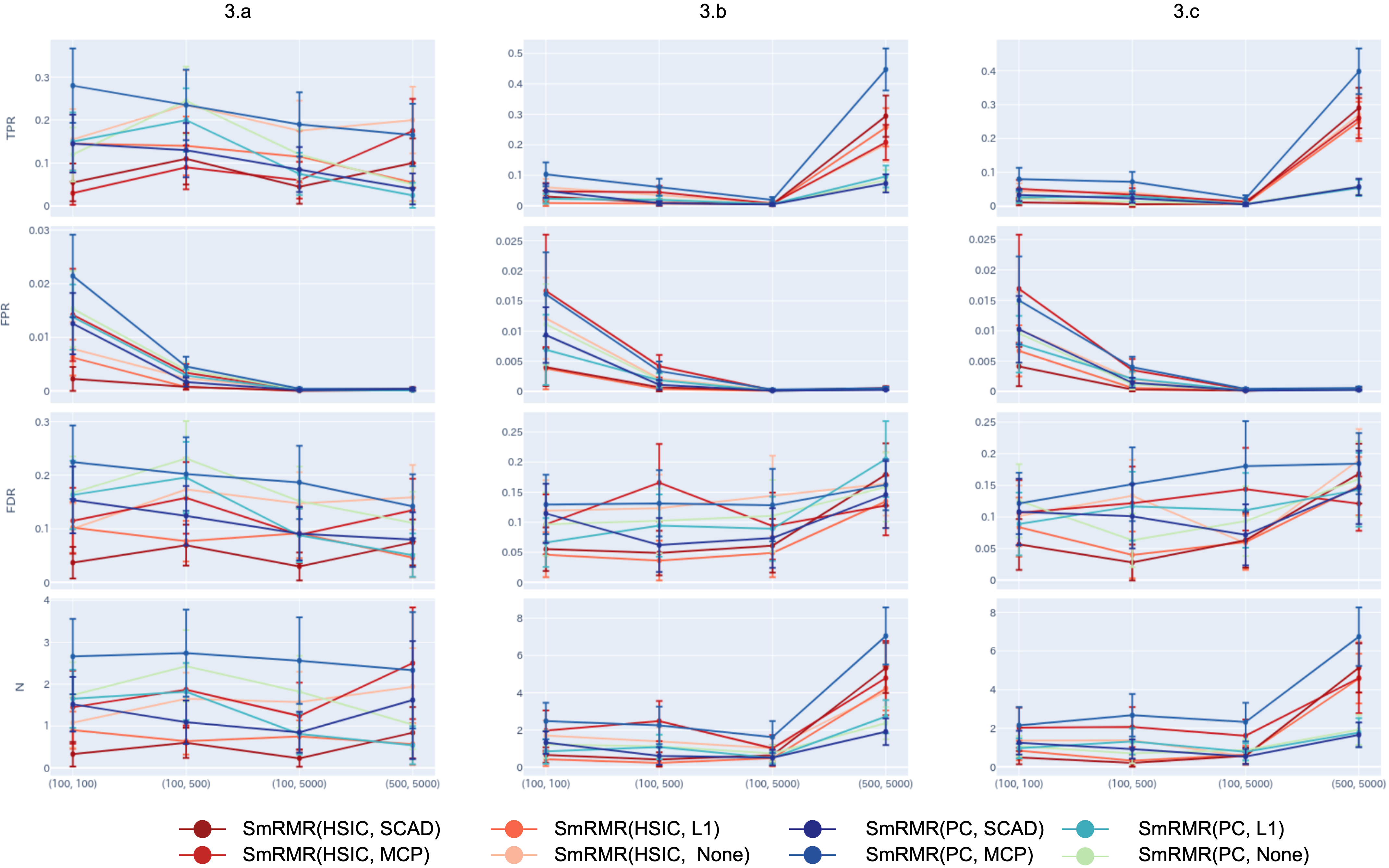}
\centering\caption{Full results for the categorical DGP with the unmodified SmRMR.}
\label{fig:catDGP_smrmr}
\end{figure}

\end{document}